\documentclass[letterpaper]{article} 
\usepackage{aaai2026}  
\usepackage{times}  
\usepackage{helvet}  
\usepackage{courier}  
\usepackage[hyphens]{url}  
\usepackage{graphicx} 
\usepackage{natbib}  
\usepackage{caption} 
\newcommand{\samethanks}[1][\value{footnote}]{\footnotemark[#1]}
\usepackage{multirow}
\usepackage[ruled,linesnumbered]{algorithm2e}
\usepackage{algpseudocode}
\usepackage{subfigure}
\usepackage{xcolor}
\usepackage[toc,title]{appendix}
\usepackage{enumitem}
\usepackage{float}
\usepackage{amsmath}
\usepackage{amsthm}
\usepackage{booktabs}
\usepackage{amssymb}
\usepackage{adjustbox}
\usepackage{makecell}
\usepackage{hyperref}

\newtheorem{definition}{Definition}
\newtheorem{theorem}{\bf Theorem}
\newtheorem{lemma}{Lemma}

\usepackage{annotate-equations}
\usepackage{booktabs}       
\usepackage{caption}        
\usepackage{amsfonts}       
\usepackage{nicefrac}       
\usepackage{microtype}      
\usepackage{enumitem}
\usepackage{bm}
\usepackage{amssymb}
\usepackage{mathtools}
\usepackage{subfigure} 
\usepackage{extarrows}
\usepackage{colortbl}
\usepackage{tikz}

\newcommand{\ie}{\textit{i.e.}}
\newcommand{\eg}{\textit{e.g.}}

\newcommand{\ourmethod}{BaCa\xspace}

\usepackage{pifont}
\newcommand{\cmark}{\color{green}\ding{51}}%
\newcommand{\xmark}{\color{red}\ding{55}}
\definecolor{lightlightgray}{gray}{0.91}
\definecolor{myblue}{RGB}{225,244,252}
\definecolor{lightgreen}{RGB}{66,210,79}
\definecolor{mygray}{RGB}{230,230,230}

\usepackage{newfloat}
\usepackage{listings}
\DeclareCaptionStyle{ruled}{labelfont=normalfont,labelsep=colon,strut=off} 
\floatstyle{ruled}
\newfloat{listing}{tb}{lst}{}
\floatname{listing}{Listing}
\title{Graph Out-of-Distribution Detection via Test-Time Calibration \\with Dual Dynamic Dictionaries}

\author{
    {Yue Hou\textsuperscript{\rm 1,\rm 2}\thanks{Equal contribution.}, Ruomei Liu\textsuperscript{\rm 1}\samethanks[1], Yingke Su\textsuperscript{\rm 2}, Junran Wu\textsuperscript{\rm 1}\thanks{Corresponding authors}, Ke Xu\textsuperscript{\rm 1,\rm 3}}
}
\affiliations{
    \textsuperscript{\rm 1}State Key Laboratory of Complex \& Critical Software Environment, Beihang University, Beijing, China\\
    \textsuperscript{\rm 2}Shen Yuan Honors College, Beihang University, Beijing, China\\
    \textsuperscript{\rm 3}Key Laboratory of Education Blockchain and Intelligent Technology,
Ministry of Education, Guangxi Normal University, China\\
    \{hou\_yue, rmliu, suyingke, wu\_junran, kexu\}@buaa.edu.cn
}

\begin{document}

\maketitle

\begin{abstract}

A key challenge in graph out-of-distribution (OOD) detection lies in the absence of ground-truth OOD samples during training. Existing methods are typically optimized to capture features within the in-distribution (ID) data and calculate OOD scores, which often limits pre-trained models from representing distributional boundaries, leading to unreliable OOD detection. Moreover, the latent structure of graph data is often governed by multiple underlying factors, which remains less explored. To address these challenges, we propose a novel test-time graph OOD detection method, termed \textbf{BaCa}, that calibrates OOD scores using dual dynamically updated dictionaries without requiring fine-tuning the pre-trained model. Specifically, BaCa estimates graphons and applies a mix-up strategy solely with test samples to generate diverse boundary-aware discriminative topologies, eliminating the need for exposing auxiliary datasets as outliers. We construct dual dynamic dictionaries via priority queues and attention mechanisms to adaptively capture latent ID and OOD representations, which are then utilized for boundary-aware OOD score calibration. To the best of our knowledge, extensive experiments on real-world datasets show that BaCa significantly outperforms existing state-of-the-art methods in OOD detection.



\end{abstract}

\section{Introduction}
With remarkable success across various domains, deep learning models are widely known to make overconfident predictions on inputs that differ from the training distribution. This often leads to misclassifying out-of-distribution (OOD) samples as in-distribution (ID) classes. 

OOD detection~\cite{usood_rc_schreyer2017detection,nlpood_zhou2021contrastive} aims to identify anomalous inputs and is essential for the safe deployment of models in open-world settings. However, performing OOD detection on graph-structured data is particularly challenging due to the non-Euclidean geometry and complex topology. 

Recent efforts~\cite{guo2023data,liu2023good,hou2025redundancyaware} in graph OOD detection fall into two main categories: \textbf{(1)} \textbf{End-to-end methods} that optimize an OOD-specific graph neural network (GNN)~\cite{gcn_kipf2017semi,gin_xu2019how} from scratch using only unlabeled ID data, and \textbf{(2)} \textbf{Post-hoc approaches}~\cite{guo2023data,wang2024goodat} that apply fine-tuned detectors on well-trained GNNs. These methods typically define an OOD score function based on the model's output logits or latent features. 
A notable extension of end-to-end training includes Outlier Exposure (OE)~\cite{junwei2024hgoe}, which leverages auxiliary OOD data during training to encourage the model to output flattened distributions for anomalous inputs. However, OE-based methods assume access to external OOD datasets, which violates the standard assumption of training solely on ID data.
Additionally, GOODAT~\cite{wang2024goodat} introduces a more practical test-time setting by directly modifying test samples without altering the pre-trained model. However, it still requires optimizing a learnable graph masker during inference, which may limit stability in real-time applications.

Despite these advancements, several notorious challenges remain underexplored. Pretrained GNNs, optimized solely on ID data, often struggle to distinguish OOD samples when their representations lie close to the ID manifold, such as when sharing similar topological structures. Moreover, the diversity of latent structural factors makes it difficult for such models to generalize well to unseen data. This limitation manifests in the form of overlapping score distributions between ID and OOD samples ($\rhd$ 
 Figure~\ref{fig:intro-1}), particularly near the decision boundary.
We argue that \textit{the key to effective test-time OOD detection lies in modeling the distributional boundary between ID and OOD samples}, especially in identifying those ambiguous cases at the boundary.

\begin{figure*}[t]
\centering
\hspace{-0.18cm}
 \subfigure[]{\label{fig:intro-1}
   \includegraphics[width=0.24\linewidth]{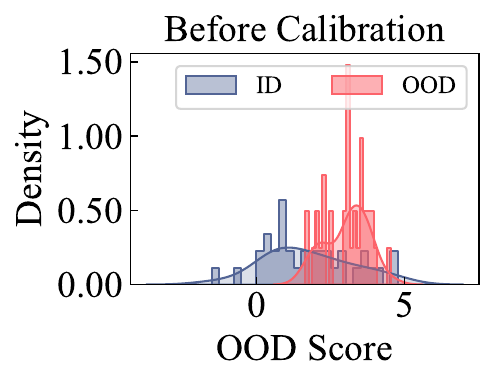}} 
\hspace{-0.18cm}
 \subfigure[]{\label{fig:intro-2}
   \includegraphics[width=0.24\linewidth]{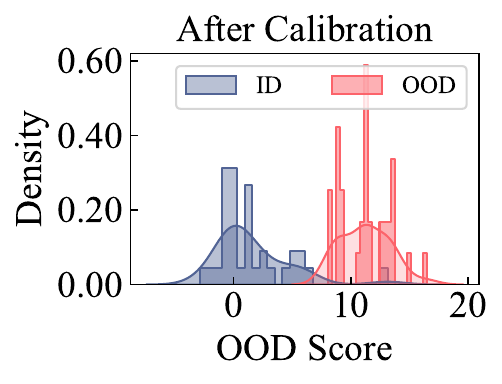}}
\hspace{-0.18cm}
 \subfigure[]{\label{fig:intro-3}
   \includegraphics[width=0.24\linewidth]{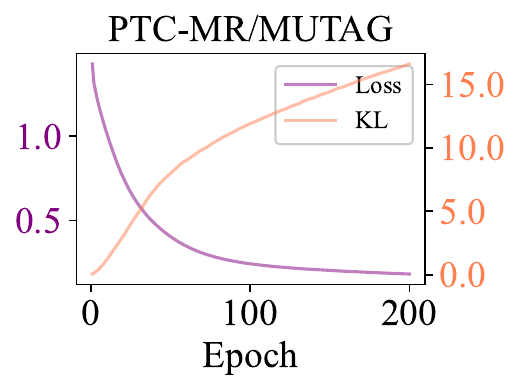}}
\hspace{-0.18cm}
 \subfigure[]{\label{fig:intro-4}
   \includegraphics[width=0.24\linewidth]{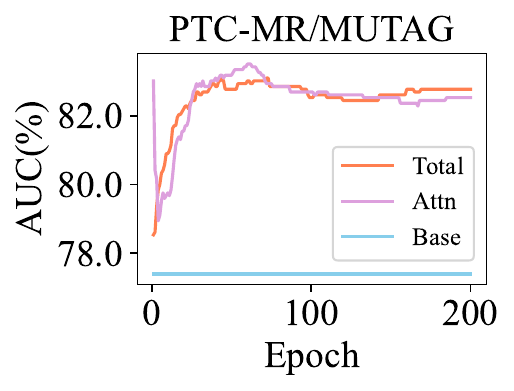}} 
\vspace{-2.5mm}
\caption{
An example of OOD score distribution and detection performance evolution over test-time iterations on the PTC/MUTAG dataset pair. 
\textbf{(a)} Before calibration, we dynamically feed the lower left tail of the OOD score distribution into the OOD dictionary and the higher right tail of the ID score distribution into the ID dictionary via two priority queues. 
\textbf{(b)} After calibration, the overlap between the ID and OOD score distributions is significantly reduced. 
\textbf{(c)} KL divergence and the loss of attention-based trainable parameters during the first 200 iterations. 
\textbf{(d)} AUC of test-time OOD detection performance over the first 200 iterations, where \textit{Total}, \textit{Attn}, and \textit{Base} denote our full method with  $S_\text{BaCa}$, attention-based calibration with $S_\text{Attn}$, and the pre-trained baseline with $S_{\text{Pre}}$, respectively.
}
\vspace{-1em}
\label{fig:intro}
\end{figure*}

Intuitively, if a test sample is more OOD-like than the least OOD sample near the ID boundary, it should be classified as OOD; similarly, if it is more ID-like than the least ID-like OOD sample, it should be treated as ID. 
Therefore, {\textbf{a natural solution}} is to calibrate OOD scores such that the overlap between ID and OOD samples is reduced ($\rhd$ Figure~\ref{fig:intro-2}), enhancing their separability at the distributional boundary.
Thus, this problem is highly challenging in:
\begin{itemize}[leftmargin=1.5em]
\item How to model the distributional boundary without relying on training ID or auxiliary OOD data?
\item How to enlarge the gap between ID and OOD data distributions through OOD score calibration?
\end{itemize}

To address these challenges, we propose a novel framework, \underline{\textbf{B}}oundary-\underline{\textbf{a}}ware \underline{\textbf{Ca}}libration for test-time graph OOD detection, termed \textbf{\ourmethod}. 
Our \ourmethod solves the aforementioned challenges and achieves adaptive OOD score calibration target through the following design.
{\textbf{Firstly}}, to model ID and OOD distributional boundaries, we perform partitioning based on initial judgment from the pre-trained model, and estimate graphons separately for ID and OOD subgroups. To capture diverse latent topological factors, we apply a graphon mixup strategy to generate synthetic samples that enhance the expressiveness of discriminative typologies and improve robustness, particularly in the early stages of detection.
{\textbf{Then}}, we propose the adaptive score calibration for the separation between ID and OOD distributions. Specifically, \ourmethod continuously collects synthetic latent representations during test time, especially those near the decision boundary, such as ID samples with OOD-like characteristics and vice versa, and dynamically inserts them into ID and OOD dictionaries maintained as priority queues. 
By incorporating a learnable attention mechanism, we adaptively calibrate OOD scores in a boundary-aware manner, reducing distributional overlap and ambiguity.
We utilize KL divergence to measure the distributional difference of OOD scores between ID and OOD samples. As iteration progresses (shown in Figure~\ref{fig:intro-3}), the KL divergence gradually increases, and the calibrated AUC consistently improves over the pre-trained encoder (see Figure~\ref{fig:intro-4}).
Extensive experiments on real-world graph datasets demonstrate the superiority of \ourmethod over state-of-the-art (SOTA) baselines. 
Notably, under the same test-time setting, \ourmethod outperforms GOODAT~\cite{wang2024goodat} on all 10 datasets, with an average AUC improvement of 8.37\%, especially on ClinTox/LIPO with gains up to 20.11\%.
Contributions of this paper are as follows:

\begin{itemize}[leftmargin=1.5em]
  \item We propose \ourmethod, a novel boundary-aware OOD score calibration framework for test-time graph OOD detection. Unlike previous approaches, it does not require prior outlier samples from auxiliary data or pre-trained model fine-tuning. 
  \item We generate diverse samples with discriminative typology and develop dual dynamic dictionaries maintained as priority queues, enabling adaptive OOD score calibration.
  \item Extensive experiments validate the effectiveness of \ourmethod, demonstrating the superior performance over SOTA baselines in unsupervised OOD detection.
\end{itemize}

\section{Notations and Preliminaries}
Before formulating the research problem, we first provide some necessary notations. Let $G=(\mathcal{V},\mathcal{E},\mathbf{X})$ represent a graph, where $\mathcal{V}$ is the set of nodes and $\mathcal{E}$ is the set of edges. The node features are represented by the feature matrix $\mathbf{X} \in \mathbb{R}^{n \times d}$, where $n=|\mathcal{V}|$ is the number of nodes and $d$ is the feature dimension. The structure information can also be described by an adjacency matrix $\mathbf{A} \in \mathbb{R}^{n \times n}$, so a graph can be alternatively represented by $G=(\mathbf{A},\mathbf{X})$. 
We summarize the frequently used notations in Appendix A.

\noindent \textbf{Test-time Graph-level OOD Detection.}
For graph-level OOD detection at test-time,
following GOODAT~\cite{wang2024goodat}, we consider an unlabeled ID dataset $\mathcal{D}^{id}=\{G_1^{id}, \cdots, G_{N_1}^{id}\}$ where graphs are sampled from distribution $\mathbb{P}^{id}$ and an OOD dataset $\mathcal{D}^{ood}=\{G_1^{ood}, \cdots, G_{N_2}^{ood}\}$ sampled from a different distribution $\mathbb{P}^{ood}$. 
Given a test sample $G$ from $\mathcal{D}^{id}_{test} \cup \mathcal{D}^{ood}_{test}$, test-time graph OOD detection aims to detect whether $G$ originates from $\mathbb{P}^{id}$ or $\mathbb{P}^{ood}$ utilizing a GNN encoder $f$ pre-trained on ID graphs $\mathcal{D}^{id}_{train} \subset \mathcal{D}^{id}$.
Specifically, the objective is to learn an OOD detector $D(\cdot,\cdot)$ that assigns an OOD detection score $S = D(f,G)$, with a higher $S$ indicating a greater probability that $G$ is from $\mathbb{P}^{ood}$ (note that $\mathcal{D}^{id}_{test} \cap \mathcal{D}^{id}_{train} = \emptyset$, $\mathcal{D}^{id}_{test} \subset \mathcal{D}^{id}$, and $\mathcal{D}^{ood}_{test} \subset \mathcal{D}^{ood}$).
It should be emphasized that graph data sourced from $\mathbb{P}^{in}$ and $\mathbb{P}^{out}$ might fall into multiple categories. However, in the unsupervised graph-level OOD task, the model is not provided with any category-specific labels.

\noindent \textbf{Graphon.} 
A graphon is a symmetric, bounded, and measurable function widely used to model the generative process of graphs~\cite{airoldi2013stochastic,lovasz2012large}. It serves as a limit object for sequences of dense graphs and captures the probability of edge existence between latent node representations in a continuous domain.
Formally, a graphon is defined as a two-dimensional symmetric Lebesgue measurable function $W: \Omega^2 \rightarrow [0, 1]$, where $\Omega$ is a probability space, typically taken as the unit interval $[0,1]$. The value $W(x,y)$ indicates the probability of an edge between two nodes associated with latent positions $x$ and $y$ in $\Omega$.
Graphons provide a principled framework for capturing the structural characteristics of graphs beyond discrete representations. By sampling latent variables from $\Omega$ and forming edges according to $W(x, y)$, one can generate synthetic graphs that share topological properties with the original graph distribution.

In real-world applications, the closed-form expression of the underlying graphon is generally unavailable and must be approximated from observed graphs. A common estimation approach is to approximate the graphon using a step function, which can be represented as a matrix $W \in [0,1]^{N \times N}$, where $N$ corresponds to the number of aligned latent positions or nodes. This matrix-form approximation enables efficient sampling of synthetic graphs and supports downstream tasks such as generation, augmentation, and structure comparison.
In this work, we adopt the USVT estimator~\cite{chatterjee2015matrix} due to its theoretical guarantees and empirical effectiveness.

\begin{figure*}[t]
    \centering
    \includegraphics[width=1\linewidth]{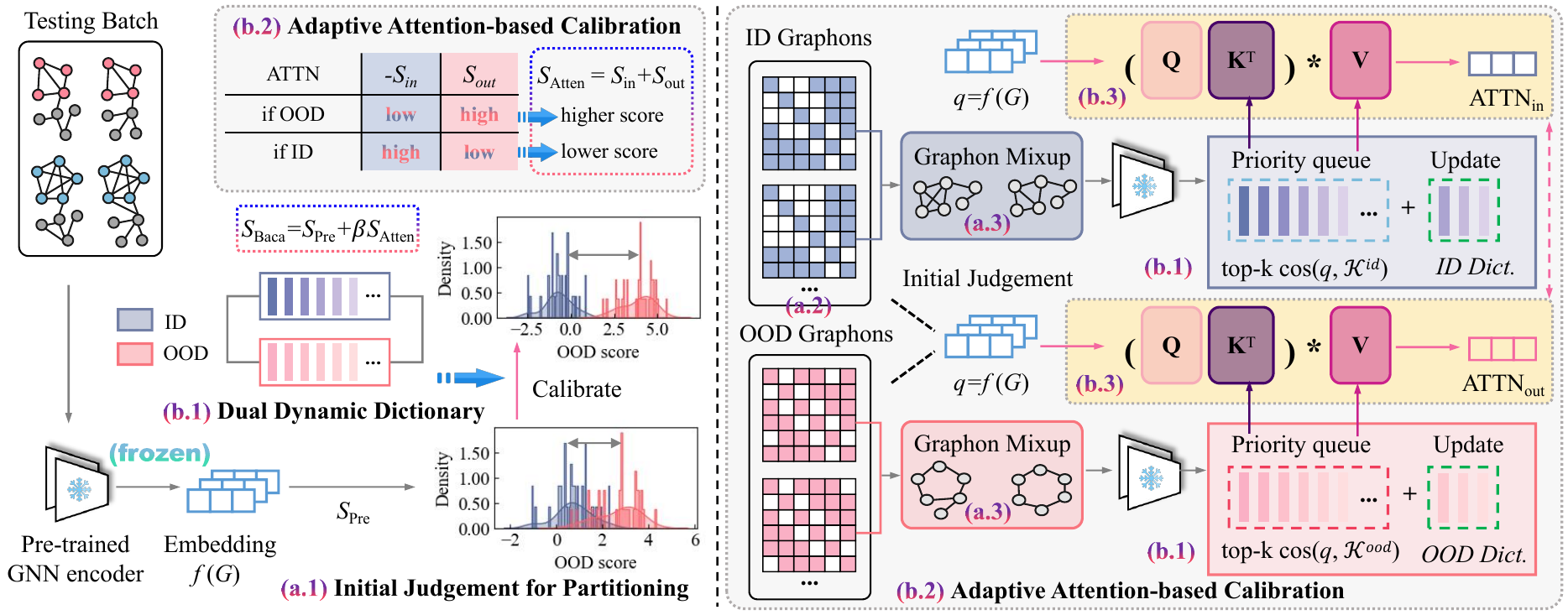}
    \hspace{20pt}
    \vspace{-0.3cm}
\caption{Overview of our proposed \ourmethod~framework. 
\textbf{(a.1)} Given a pre-trained GNN encoder and test samples, we first compute the initial OOD scores and partition the samples into two preliminary subgroups based on the pre-trained model's predictions. 
\textbf{(a.2–a.3)} Within each subgroup, diverse discriminative typologies are generated via graphon mixup and stored in dual dynamic dictionaries maintained as priority queues. 
\textbf{(b.1–b.2)} The priority queue–based dictionaries are used to support adaptive, attention-based score calibration. 
\textbf{(b.3)} The adaptive attention module is optimized during inference to compute the final calibrated OOD score.
}
\label{fig:framework}
\vspace{-0.3cm}
\end{figure*}

\section{Methodology}
In this section, we elaborate on the proposed adaptive redundancy-aware OOD score calibration for test-time graph OOD detection, termed \textbf{\ourmethod}.

\subsection{Overall Framework}
In general, the basic objective in OOD detection for obtaining a GNN encoder $f$ is defined as:
\begin{equation}
\min_f \mathbb{E}_{G \sim \mathcal{D}^{\text{in}}_{\text{train}}} \mathcal{L}_{\text{Pre}}(f;G),
\label{eq:pre}
\end{equation}
where $\mathcal{L}_{\text{Pre}}$ denotes the pretraining loss function. For end-to-end OOD detection methods~\cite{liu2023good}, the OOD score of a test sample is typically derived directly from the output of this pre-trained model. However, the initial judgment made by the pre-trained model regarding a sample's distribution may be unreliable, due to its lack of exposure to true OOD samples. This can lead to inaccurate OOD scores, especially near the boundary between ID and OOD distributions.

To enable test-time OOD score calibration without updating the pre-trained model, we identify two key challenges: \textbf{(C1)} how to effectively model the boundary between ID and OOD samples, and \textbf{(C2)} how to design a robust score calibration mechanism.
To address \textbf{(C1)}, we first partition test samples into two groups based on the initial score estimation, and then estimate graphons separately for each group. A graphon mixup strategy is applied within each group to generate diverse discriminative typologies that enhance the representation of boundary distributions.
To address \textbf{(C2)}, we maintain dual dynamic dictionaries using priority queues and perform adaptive score calibration via attention mechanisms.
The overall pipeline of \ourmethod~is illustrated in Figure~\ref{fig:framework}.

\subsection{Boundary-Aware Latent Pattern Modeling}

\noindent \textbf{Subgroup Partitioning Based on Initial Judgment.}
We utilize the pre-trained model $f$ to extract the representation of each test sample $G \in \mathcal{D}_{\text{test}}$ and compute its initial OOD score $S_{\text{Pre}} = \mathcal{L}_{\text{Pre}}(f; G)$ using Eq.~\eqref{eq:pre}. This score serves as an initial judgment of the sample’s distributional status.

\noindent \textbf{Graphon Estimation for Latent Factor Construction.}
To capture the structural differences among test samples and model their distributional variation, we employ graphons to estimate the characteristic topologies in different subsets of graphs.
A graphon $W: \Omega^2 \rightarrow [0, 1]$ defines the probability of edge existence between any two latent positions sampled from a base space $\Omega$. Given a graphon, a random graph can be generated as follows:
\begin{equation}
\begin{aligned}
v_n &\sim \text{Uniform}(\Omega), \quad \text{for } n = 1, \dots, N, \\
a_{nn'} &\sim \text{Bernoulli}(W(v_n, v_{n'})), \quad \text{for } n, n' = 1, \dots, N,
\label{eq:bern}
\end{aligned}
\end{equation}
where $v_n$ denotes the latent position of node $n$, and $a_{nn'}$ indicates whether an edge exists between nodes $n$ and $n'$. This process results in an adjacency matrix $\mathbf{A} \in \{0,1\}^{N \times N}$, which defines the structure of a sampled graph $\tilde{G}(\tilde{\mathcal{V}}, \tilde{\mathcal{E}})$ with $\tilde{\mathcal{V}} = \{1, \dots, N\}$ and $\tilde{\mathcal{E}} = \{(n, n') \mid a_{nn'} = 1\}$.

Since the true graphon is an unknown function and cannot be recovered in closed form, we adopt the step-function approximation commonly used in prior work~\cite{chatterjee2015matrix,xu2021learning,yuan2025graver}. A step-function graphon $W^P: [0,1]^2 \rightarrow [0,1]$ is expressed as: $W^P(x, y) = \sum_{n,n'=1}^{N} w_{nn'} \, \mathbb{1}_{\mathcal{P}_n \times \mathcal{P}_{n'}}(x, y)$,
where $\mathcal{P} = (\mathcal{P}_1, \dots, \mathcal{P}_N)$ is a uniform partition of $[0,1]$ into $N$ intervals, and $w_{nn'} \in [0,1]$ represents the estimated connection probability between intervals $\mathcal{P}_n$ and $\mathcal{P}_{n'}$. The indicator function $\mathbb{1}_{\mathcal{P}_n \times \mathcal{P}_{n'}}(x,y)$ equals 1 if $(x, y) \in \mathcal{P}_n \times \mathcal{P}_{n'}$ and 0 otherwise.

Based on $S_{\text{Pre}}$, we partition all samples in the current test-time batch into two mutually exclusive subsets:$\mathcal{D}^{\text{batch}}_{\text{test}} = \mathcal{C}^{\text{id}} \cup \mathcal{C}^{\text{ood}}$,
where $\mathcal{C}^{\text{ood}} = \{ W_{i,m} \}_{\tilde{y}=1}^{M}$ and $\mathcal{C}^{\text{id}} = \{ W_{i,m'} \}_{\tilde{y}=0}^{M'}$, with $M$ and $M'$ denoting the number of samples initially predicted as OOD and ID, respectively. This partitioning allows ID and OOD candidate samples to be processed separately during downstream graphon mixup and dictionary construction, relying only on the pre-trained model and soft predictions, without requiring ground-truth supervision.

\noindent \textbf{Graphon Mixup for Discriminative Typology Expansion.}
After the partitioning step, test-time samples are divided into two disjoint subgroups and separate sets of graphons are estimated to model the structural patterns within each group. However, the discriminative topological factors responsible for distributional differences are often multifaceted rather than governed by a single mode. Moreover, the estimated graphons may not sufficiently capture structures near the boundary regions, leading to unstable detection and poor generalization, especially in early-stage inference.

To alleviate this issue, we propose a graphon-level mixup strategy performed within each subgroup (i.e., among ID graphons and among OOD graphons separately). This approach interpolates between graphons derived from structurally distinct samples within the same class, thereby enhancing internal structural diversity and enriching the boundary space.
Formally, let $W_i$ and $W_j$ be two graphons estimated from the same group (e.g., $\mathcal{C}^{\text{ood}}$). We define their mixed graphon as:
\begin{equation}
W_s = \lambda W_i + (1 - \lambda) W_j, \quad \lambda \in [0, 1],
\label{eq:mixup}
\end{equation}
where $\lambda$ is a balancing hyperparameter. The resulting $W_s$ lies in the convex hull of $W_i$ and $W_j$ and can be interpreted as a new generative process that inherits structural traits from both sources. Sampling from $W_s$ generates graphs located in the interpolated region between the two subgroups, which helps bridge discontinuities in the estimated structure space and populate low-density zones near the ID/OOD boundary.
To formalize this notion, we introduce the concept of a \emph{discriminative typology}, which characterizes the essential structural properties that determine a graph’s subgroup membership.

\begin{definition}[\textbf{Discriminative Typology}]
Given a graph $G$, a discriminative typology $T_G$ is a structural pattern that reflects the most representative and characteristic features of $G$ with respect to its latent distribution,~\ie, ID or OOD.
\label{def:1}
\end{definition}
\noindent Intuitively, typologies summarize structural traits that differentiate subgroups within a distribution and intuitively capture the generative semantics of graph samples. Our hypothesis is that graphons estimated from a group of graphs encode their typological characteristics, and linear combinations of such graphons preserve essential features from the source groups.

\begin{theorem} \label{thm:mix}
Let $W_G$ and $W_H$ be graphons estimated from two subgroups $G$ and $H$ of the same distribution type (\ie, both ID or both OOD). Let the interpolated graphon be defined as $W_s = \lambda W_G + (1 - \lambda) W_H$, where $\lambda \in [0,1]$. Then, for any discriminative typology $T_G$ and $T_H$:
\begin{equation}
\begin{aligned}
    \left| t(T_G, W_s) - t(T_G, W_G) \right| &\leq (1 - \lambda) \cdot \delta_{GH}, \\
    \left| t(T_H, W_s) - t(T_H, W_H) \right| &\leq \lambda \cdot \delta_{GH},
\end{aligned}
\end{equation}
where $t(F, W) = \int_{[0,1]^{|\tilde{\mathcal{V}}|}} \prod_{(i, j) \in {|\tilde{\mathcal{E}}|}} W(x_i, x_j) \, \prod_{i \in {|\tilde{\mathcal{V}}|}} dx_i$ denotes the homomorphism density of structure $F$ in graphon $W$, and $\delta_{GH} = \| W_G - W_H \|_{\square}$ is the cut norm distance between $W_G$ and $W_H$. 
The detailed proof is in Appendix D.
\label{thm:1}
\end{theorem}

\textbf{Remark:} The theorem indicates that the mixed graphon $W_s$ retains the key structural characteristics from both $W_G$ and $W_H$, with bounded deviations depending on the mixing ratio $\lambda$ and the structural dissimilarity between the original subgroups. Since $W_G$ and $W_H$ originate from the same distribution (either ID or OOD), the synthetic graphs sampled from $W_s$ remain typologically consistent with their source distribution, enabling meaningful boundary exploration without compromising distributional integrity.
Through this graphon-level mixup procedure, we can generate an arbitrary number of graphs at test-time, filling in the low-density regions between known ID and OOD modes and improving the robustness of boundary estimation.

\noindent\textbf{Random Size Sampling for Boundary Diversity.}
To enhance structural diversity and better approximate the true distributional variability among graphs, we introduce a random size-based sampling strategy. Although an interpolated graphon $W_s \in [0,1]^{N \times N}$ allows infinite graph generation, naive sampling typically results in graphs of size close to $N$, limiting diversity.
To mitigate this, we randomly select a target size $r \in [2, N]$ and generate the graph from sampled graphon $W_{s}' \in [0,1]^{r \times r}$. The existence of an edge between nodes $i$ and $j$ is determined by sampling from a Bernoulli distribution in Eq.~\eqref{eq:bern} with the parameter $W_{s}'(i, j)$.

\subsection{Adaptive Calibration via Dual Dynamic Dictionary}

\noindent \textbf{Dual Priority Queues for Dynamic Dictionary Maintenance.}
As test-time samples arrive in successive batches, the boundary between ID and OOD samples evolves dynamically. To adaptively track this boundary, we maintain two separate dynamic dictionaries for ID and OOD samples, each implemented as a priority queue. These dictionaries are updated online according to the sample’s relative position to the ID/OOD boundary, as estimated from the OOD score.

Intuitively, if a sample is more OOD-like than the least OOD sample (\ie, on the boundary side of the OOD distribution), it is added to the OOD dictionary; similarly, if a sample is more ID-like than the least ID-like sample, it is added to the ID dictionary. In practice, this means that the right tail of the ID score distribution, where ID samples are most similar to OOD, is inserted into the ID dictionary, while the left tail of the OOD score distribution, where OOD samples are most similar to ID, is inserted into the OOD dictionary. We refer to these as \textit{latent ID features} and \textit{latent OOD features}, respectively, as they represent boundary-side discriminative typologies.
The initial dictionaries are constructed based on the pre-trained model's score, and as test-time progresses, these dictionaries are continuously enriched by newly generated synthetic samples from graphon mixup, which increases the diversity of latent patterns near the boundary.



During inference, the ID and OOD dictionaries are maintained as fixed-length priority queues. This design allows encoded features from previous batches to be reused, decoupling the dictionary size from the mini-batch size. The queue size $l$ is a tunable hyperparameter and enables storage of more diverse and representative structures. Taking the OOD dictionary as an example, we denote it as $\mathcal{K}^{\text{ood}}_{l} = \{ k_1^{\text{ood}}, k_2^{\text{ood}}, \dots, k^{\text{ood}}_{l'} \}$ with $l \geq l'$. New candidates are added to the queue only if their OOD score exceeds that of the front element. In this setup, the front of the OOD queue always corresponds to the sample closest to the ID/OOD boundary.
Similarly, we maintain the ID dictionary $\mathcal{K}^{\text{id}}_{l}$ using the same mechanism, where the front represents the least ID-like inlier.

In summary, we dynamically feed the lower left tail of the OOD score distribution into the OOD dictionary, and the higher right tail of the ID score distribution into the ID dictionary. This dual-priority-queue mechanism ensures that both dictionaries retain the most representative and boundary-sensitive graphon-derived features, allowing for adaptive and efficient modeling of the evolving ID/OOD structure during test time.

\noindent \textbf{Adaptive Attention-based Score Calibration.}
To enhance calibration adaptively to capture boundary-aware representations, we introduce an attention mechanism over the ID and OOD dictionaries. Since attention scores are often concentrated on a small subset of keys, we compute attention over only the top-$\mathbb{K}$ most relevant entries, improving efficiency and reducing noise from irrelevant matches.
Taking OOD dictionary as an example, we first derive the query $q = f(G)$ for a test sample $G \in \mathcal{D}_{test}$ and compute the cosine similarity $\cos (k^{ood}_i, q)$ with each key $k^{ood}_i$ in OOD dictionary $\mathcal{K}^{ood}_{n^{\prime}}$. Then, we denote the sorted list of these similarities in ascending order as \(\cos(k^{ood}_{(1)}, q) \leq \cos(k^{ood}_{(2)}, q) \leq \cdots \leq \cos(k^{ood}_{(n^{\prime})}, q)\).
The top $\mathbb{K}$ entries are selected to form the candidate set $\mathcal{\hat{K}}^{\text{ood}}_{(:\mathbb{K})}$.
We construct the attention components as:
\begin{equation}
\begin{aligned}
\mathbf{Q} = q\mathbf{W}_Q, 
\mathbf{K} = \mathcal{\hat{K}}^{\text{ood}}_{(:\mathbb{K})} \mathbf{W}_K, 
\mathbf{V} = \mathcal{\hat{K}}^{\text{ood}}_{(:\mathbb{K})} \mathbf{W}_V, \\
\text{ATTN}_{\text{out}}(\mathbf{Q}, \mathbf{K}, \mathbf{V}) = \text{softmax}( \frac{\mathbf{Q} \mathbf{K}^\top}{\sqrt{d}} ) \mathbf{V},
\end{aligned}
\end{equation}
where $\mathbf{W}_Q \in \mathbb{R}^{d \times d}$ and $\mathbf{W}_K, \mathbf{W}_V \in \mathbb{R}^{\mathbb{K} \times d}$ are learnable matrices.
The calibrated OOD score based on OOD dictionary is then defined as:
\begin{equation} \label{eq:s-out}
S_{\text{out}}(G) = \text{ATTN}_{\text{out}}(\mathbf{Q}, \mathbf{K}, \mathbf{V}).
\end{equation}
The complete OOD dictionary includes both the priority queue and memory bank:
$\mathcal{K}^{\text{ood}}_{\text{total}} = \mathcal{K}^{\text{ood}}_l \cup \mathcal{K}^{\text{ood}}_{\text{mb}}$, where $\mathcal{K}^{\text{ood}}_{\text{mb}}$ denotes a fixed-size memory buffer.
Similarly, we calculate the negative cosine similarity between the query and each key in the ID dictionary:
\begin{equation} \label{eq:s-in}
S_{\text{in}}(G) = -\text{ATTN}_{\text{in}}(\mathbf{Q}, \mathbf{K}, \mathbf{V}),
\end{equation}
where $\mathbb{K}$-th largest cosine similarity is selected, and the ID dictionary is composed as $\mathcal{K}^{id}_{total} = \mathcal{K}^{id}_l \cup \mathcal{K}^{id}_{mb}$.
The final boundary-aware calibrated score is then given by:
\begin{equation} \label{eq:s-atten}
S_\text{Atten}(G) = S_{\text{in}}(G) + S_{\text{out}}(G),
\end{equation}
If $G$ is an ID sample, it will typically have high similarity with the ID dictionary and low similarity with the OOD dictionary, resulting in a low $S_{\text{Attn}}(G)$. Conversely, OOD samples yield higher values.
This calibration mechanism encourages a clearer separation of score distributions between ID and OOD samples by modeling diverse features and structural boundaries.
We integrate $S_{Attn}(G)$ into the overall objective:
\begin{equation} \label{eq:s-baca}
S_\text{BaCa} = S_{\text{Pre}} + \beta \cdot S_\text{Attn}(G),
\end{equation}
where $\beta$ is a trade-off hyperparameter controlling the influence of test-time similarity calibration.

\noindent \textbf{Training Objective.} 
To optimize the learnable parameters $\mathbf{W}_Q$, $\mathbf{W}_K$, and $\mathbf{W}_V$, we employ a dual binary cross-entropy loss that supervises the attention-based similarity scores.
Formally, the training objective is defined as:
\begin{equation}
\begin{aligned}
\label{eq:loss}
\mathcal{L} = 
&- \mathbb{E}_{\mathcal{K}^{id}} 
\left[ \log(\text{ATTN}_{\text{in}}) + \log(1-\text{ATTN}_{\text{out}})\right] \\
&- \mathbb{E}_{\mathcal{K}^{ood}}
\left[ \log(1-\text{ATTN}_{\text{in}}) + \log(\text{ATTN}_{\text{out}})\right].
\end{aligned}
\end{equation}
This loss encourages ID samples to yield high attention scores with ID dictionary and low scores with OOD dictionary, while OOD samples are trained to exhibit the opposite pattern.

\noindent\textbf{Computational Complexity Analysis.} 
For graph generation, suppose we aim to generate $l$ graphs with $N$ nodes. The complexity is $\mathcal{O}(lN)$ for node sampling and $\mathcal{O}(lN^2)$ for edge construction, resulting in a total complexity of $\mathcal{O}(lN^2)$.
For dynamic dictionary construction, \ourmethod~relies solely on dot-product operations between test-time samples and stored entries. This is equivalent to adding a linear transformation layer, introducing a per-sample complexity of $\mathcal{O}(dl)$, where $d$ is the feature dimension and $l$ denotes the priority queue size. Updating the priority queue has a complexity of $\mathcal{O}(\log l)$ per insertion.
For the attention-based score calibration, given query $\mathbf{Q} \in \mathbb{R}^{1 \times d}$ and key-value matrices $\mathbf{K}, \mathbf{V} \in \mathbb{R}^{\mathbb{K} \times d}$ from the top-$\mathbb{K}$ dictionary entries, the main computation involves $\mathbf{QK}^\top \in \mathbb{R}^{1 \times \mathbb{K}}$ and its softmax weighting over $\mathbf{V}$, resulting in $\mathcal{O}(2\mathbb{K}d)$ complexity per test sample. Since $\mathbb{K}$ is typically small, this module introduces negligible overhead and scales well during inference.

\begin{table*}[t]
\centering
\vspace{-0.2cm}
\resizebox{1\textwidth}{!}{
\begin{tabular}{l | cccccccccc|c}
\toprule
ID dataset & BZR & PTC-MR & AIDS & ENZYMES & IMDB-M & Tox21 & FreeSolv & BBBP & ClinTox & Esol & \multirow{2}{1.8em}{A.A.} \\
OOD dataset & COX2 & MUTAG & DHFR & PROTEIN & IMDB-B & SIDER & ToxCast & BACE & LIPO & MUV \\
\midrule
\multicolumn{12}{l}{\textbf{Graph Kernel Based Methods}} \\
PK-LOF      & 42.22±8.39 & 51.04±6.04 & 50.15±3.29 & 50.47±2.87 & 48.03±2.53 & 51.33±1.81 & 49.16±3.70 & 53.10±2.07 & 50.00±2.17 & 50.82±1.48 & 49.63 \\
PK-OCSVM    & 42.55±8.26 & 49.71±6.58 & 50.17±3.30 & 50.46±2.78 & 48.07±2.41 & 51.33±1.81 & 48.82±3.29 & 53.05±2.10 & 50.06±2.19 & 51.00±1.33 & 49.52 \\
PK-iF & 51.46±1.62 & 54.29±4.33 & 51.10±1.43 & 51.67±2.69 & 50.67±2.47 & 49.87±0.82 & 52.28±1.87 & 51.47±1.33 & 50.81±1.10 & 50.85±3.51 &51.45  \\
WL-LOF   & 48.99±6.20 & 53.31±8.98 & 50.77±2.87 & 52.66±2.47 & 52.28±4.50 & 51.92±1.58 & 51.47±4.23 & 52.80±1.91 & 51.29±3.40 & 51.26±1.31 & 51.68  \\
WL-OCSVM    & 49.16±4.51 & 53.31±7.57 & 50.98±2.71 & 51.77±2.21 & 51.38±2.39 & 51.08±1.46 & 50.38±3.81 & 52.85±2.00 & 50.77±3.69 & 50.97±1.65 & 51.27 \\
WL-iF  & 50.24±2.49 & 51.43±2.02 & 50.10±0.44 & 51.17±2.01 & 51.07±2.25 & 50.25±0.96 & 52.60±2.38 & 50.78±0.75 & 50.41±2.17 & 50.61±1.96 & 50.87 \\
\midrule
\multicolumn{12}{l}{\textbf{Anomaly Detection Methods}} \\
OCGIN  & 76.66±4.17 & 80.38±6.84 & 86.01±6.59 & 57.65±2.96 & 67.93±3.86 & 46.09±1.66 & 59.60±4.78 & 61.21±8.12 & 49.13±4.13 & 54.04±5.50 & 63.87 \\
GLocalKD  & 75.75±5.99 & 70.63±3.54 & 93.67±1.24 & 57.18±2.03 & 78.25±4.35 & 66.28±0.98 & 64.82±3.31 & 73.15±1.26 & 55.71±3.81 & 86.83±2.35 & 72.23 \\
\midrule

\multicolumn{12}{l}{\textbf{Self-supervised Training Methods}} \\
InfoGraph-iF & 63.17±9.74 & 51.43±5.19 & 93.10±1.35 & 60.00±1.83 & 58.73±1.96 & 56.28±0.81 & 56.92±1.69 & 53.68±2.90 & 48.51±1.87 & 54.16±5.14 & 59.60 \\
InfoGraph-MD & 86.14±6.77 & 50.79±8.49 & 69.02±11.67 & 55.25±3.51 & \underline{81.38±1.14} & 59.97±2.06 & 58.05±5.46 & 70.49±4.63 & 48.12±5.72 & 77.57±1.69 & 65.68  \\
GraphCL-iF & 60.00±3.81 & 50.86±4.30 & 92.90±1.21 & 61.33±2.27 & 59.67±1.65 & 56.81±0.97 & 55.55±2.71 & 59.41±3.58 & 47.84±0.92 & 62.12±4.01 & 60.65  \\
GraphCL-MD & 83.64±6.00 & 73.03±2.38 & 93.75±2.13 & 52.87±6.11 & 79.09±2.73 & 58.30±1.52 & 60.31±5.24 & 75.72±1.54 & 51.58±3.64 & 78.73±1.40 & 70.70 \\
 
GOOD-D & \underline{93.00±3.20} & 78.43±2.67 & 98.91±0.41 & {61.89±2.51} & 79.71±1.19 & 65.30±1.27 & 70.48±2.75 & 81.56±1.97 & 66.13±2.98 & 91.39±0.46 & \underline{78.68} \\

HGOE & $-$ & $-$ & \underline{99.28±0.34} & 64.44±2.19 & \textbf{81.74±2.25} & 68.24±0.60 & \underline{82.89±2.33} & \underline{83.46±1.79} & \underline{70.09±1.52} & \underline{92.64±2.44} & $-$ \\

\midrule
\multicolumn{12}{l}{\textbf{Test-time and Data-centric Methods}} \\

AAGOD-GIN\ensuremath{_S+} & 76.75 & $-$ & $-$ & 66.22 & 59.00 & 64.26  & $-$ & 67.80  & $-$ & $-$ & $-$ \\
AAGOD-GIN\ensuremath{_L+} & 76.00 & $-$ & $-$ & 65.89 & 62.70 & 57.59  & $-$ & 57.13  & $-$ & $-$ & $-$ \\
GOODAT & 82.16{±0.15} & \underline{81.84{±0.57}} & {96.43{±0.25}} & \underline{66.29{±1.54}} & 79.03{±0.03} & \underline{68.92{±0.01}} & {68.83{±0.02}} & {77.07{±0.03}} & {62.46{±0.54}} & {85.91{±0.27}} & 76.89  \\

\midrule 
\rowcolor{lightlightgray}\ourmethod & \textbf{94.23±0.42} & \textbf{86.53±1.39} & \textbf{99.86±0.03} & \textbf{67.10±1.43} & 80.93±0.69& \textbf{69.82±0.59} & \textbf{83.12±0.42} & \textbf{93.11±0.29} & \textbf{82.57±0.23} & \textbf{95.31±0.14} & \textbf{85.26} \\

{Improve} & \textcolor{lightgreen}{$\triangle$+1.23} & \textcolor{lightgreen}{$\triangle$+4.69} & \textcolor{lightgreen}{$\triangle$+0.58} & \textcolor{lightgreen}{$\triangle$+0.81} & \textcolor{orange}{$\nabla$-0.81}& \textcolor{lightgreen}{$\triangle$+0.90} & \textcolor{lightgreen}{$\triangle$+0.23} & \textcolor{lightgreen}{$\triangle$+9.65} & \textcolor{lightgreen}{$\triangle$+12.48} & \textcolor{lightgreen}{$\triangle$+2.67} & \textcolor{lightgreen}{$\triangle$+6.58} \\

\bottomrule
\end{tabular}}
\caption{OOD detection results in terms of AUC ($\%$, mean $\pm$ std). The best and runner-up results are highlighted with \textbf{bold} and \underline{underline}, respectively. A.A. is short for average AUC. The results of baselines are derived from the published works, with unreported results denoted by `$-$'.} 

\label{tab:main}
\vspace{-0.3cm}
\end{table*}

\begin{table}[t]
\centering
\resizebox{\linewidth}{!}{
\begin{tabular}{cc | cccc}
\toprule
\multirow{2}{*}{\textit{ID Dict.}} & \multirow{2}{*}{\textit{OOD Dict.}} & BZR & PTC-MR & AIDS & ENZYMES\\
&   & COX2 & MUTAG & DHFR & PROTEIN\\
\midrule
\xmark & \xmark & 92.95±0.15& 77.59±4.37 & 99.24±0.06 & 63.14±0.00 \\

\xmark & \cmark & 93.22±0.12 & \underline{85.71±1.88} & \underline{99.80±0.04} & 65.51±2.39 \\

\cmark & \xmark  & \underline{93.66±0.03} & 84.65±2.45 & 99.50±0.01 & \underline{66.30±2.43} \\

\midrule
\cmark & \cmark
& \textbf{94.23±0.42} & \textbf{86.53±1.39} & \textbf{99.86±0.03} & \textbf{67.10±1.43} \\
\bottomrule
\end{tabular}}
\caption{Ablation study results of \ourmethod and its variants in terms of AUC ($\%$, mean $\pm$ std).} 
\label{tab:ablation-half}
\vspace{-0.5cm}
\end{table}

\vspace{-2.0mm}
\section{Experiment}

In this section, we empirically evaluate the effectiveness of the proposed \ourmethod.
\footnote{The code of \ourmethod is available at {https://github.com/name-is-what/BaCa}.} 

\noindent \textbf{Datasets.}
For OOD detection, we employ 10 pairs of datasets from two mainstream graph data benchmarks (i.e., TUDataset~\cite{tu_Morris2020} and OGB~\cite{ogb_hu2020open}) following GOOD-D~\cite{liu2023good}. 
Each pair of datasets belongs to the same field and shares similar features, but exhibits distribution shifts between the two datasets in the pair. 
Further details are shown in Appendix E.1.

\noindent \textbf{Baselines.}
We compare \ourmethod with a wide range of graph OOD detection baselines, grouped into the following categories: (1) \textbf{graph kernel based methods}~\cite{pk_neumann2016propagation,wlgk_shervashidze2011weisfeiler}, (2) \textbf{anomaly detection methods}~\cite{glocalkd_ma2022deep,ocgin_zhao2021using}, (3) \textbf{self-supervised methods}~\cite{infograph_sun2020infograph,graphcl_you2020graph}, and (4) \textbf{test-time and data-centric methods}~\cite{guo2023data,wang2024goodat,junwei2024hgoe}.

\noindent \textbf{Evaluation and Implementation.}
We evaluate \ourmethod with a popular OOD detection metric, i.e., area under receiver operating characteristic Curve (AUC). Higher AUC values indicate better performance.  
The reported results are the mean performance with standard deviation after 5 runs.
We perform grid search to select the key hyper-parameters of \ourmethod.
During ID and OOD graphon mixup, \(\lambda\) was randomly chosen from the range \([0.01,1.0]\). 
Our \ourmethod is instantiated on top of the well-trained 5-layer GIN~\cite{gin_xu2019how}, and improves its OOD detection performance in a fully post-hoc and test-time setting without model update or auxiliary data.

\noindent \textbf{Performance on OOD Detection.}
We compare \ourmethod with representative baselines on graph OOD detection tasks in Table~\ref{tab:main}. \ourmethod achieves the best performance on 7 out of 10 dataset pairs, and runner-up performance on two others. Compared with end-to-end baselines such as GOOD-D~\cite{liu2023good} and HGOE~\cite{junwei2024hgoe}, our method consistently yields higher detection accuracy. 
Notably, under the same test-time setting, \ourmethod outperforms GOODAT~\cite{wang2024goodat} on all 10 datasets, with an average AUC improvement of 8.37\%.
We also observe that both GOODAT and \ourmethod perform relatively poorly on the IMDB-M/IMDB-B pair. This is likely due to their structural similarity, as both originate from the same dataset source. 
Further analysis with case study are provided in Appendix E.5.

\noindent \textbf{Ablation Study.}
We perform ablation studies by selectively removing the ID dictionary and OOD dictionary (denoted as \textit{ID Dict.} and \textit{OOD Dict.}, respectively). The results are summarized in Table~\ref{tab:ablation-half}.
We first observe that \ourmethod with both dictionaries (last row) consistently achieves the best performance across all dataset pairs, highlighting the effectiveness of our dual-dictionary design. The first row corresponds to removing both dictionaries, which reduces the model to the pretrained baseline without score calibration. Notably, using only one of the two dictionaries (either ID or OOD) leads to a clear drop in performance, indicating that both are necessary to enable accurate boundary-aware score calibration.

\begin{figure}[!t]
\begin{center}
\centerline{\includegraphics[width=0.97\linewidth]{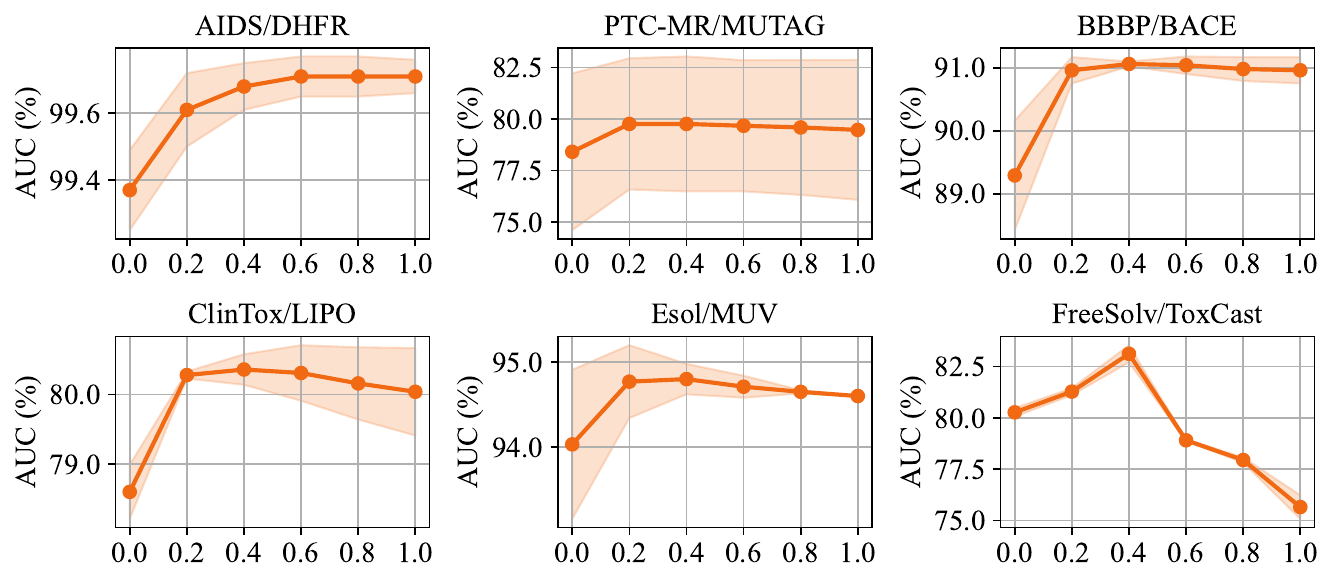}}
\vspace{-0.2cm}
\caption{The sensitivity of $\beta$ on calibration.}
\label{fig:sensi-b}
\end{center}
\vspace{-0.6cm}
\end{figure}

\begin{figure}[!t]
\begin{center}
\centerline{\includegraphics[width=0.97\linewidth]{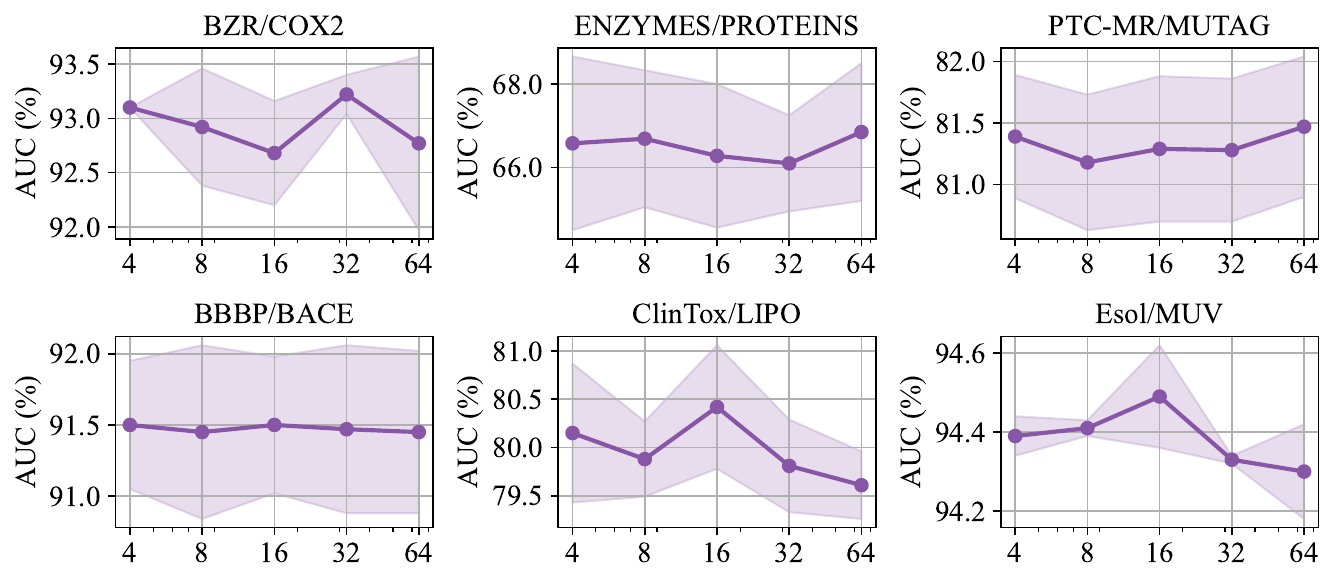}}
\vspace{-0.2cm}
\caption{The sensitivity of $\mathbb{K}$ on calibration.}
\label{fig:sensi-k}
\end{center}
\vspace{-0.8cm}
\end{figure}


\noindent \textbf{Sensitivity Analysis of $\beta$.}
We further study the effect of $\beta$, the weight assigned to the calibration term in the final score. As shown in Figure~\ref{fig:sensi-b}, we vary $\beta$ from 0.1 to 1.0. While performance is relatively stable in a mid-range band, too small or too large values of $\beta$ may suppress or over-amplify the influence of similarity-based score correction. These results confirm that \ourmethod is robust to reasonable choices of $\beta$, but tuning remains important for optimal performance.

\noindent \textbf{Sensitivity Analysis of $\mathbb{K}$.} 
To analyze the sensitivity of $\mathbb{K}$ for \ourmethod, we alter the value from 4 to 64. The AUC w.r.t different selections of $\mathbb{K}$ is plotted in Figure~\ref{fig:sensi-k}. 
Results demonstrate the performance is sensitive to changes in $\mathbb{K}$ and contains a reasonable range across different datasets.

\begin{table}[ht]
\centering
\resizebox{\linewidth}{!}{
        \begin{tabular}{c|ccccc}
            \toprule
            \multirow{2}{*}{$\lambda$} & AIDS & BZR  & PTC-MR & Esol & ClinTox   \\
            & DHFR & COX2 & MUTAG & MUV & LIPO     \\
            \midrule
            $[0.01,0.2]$ & 99.83±0.04 & 92.89±0.33 & 85.63±1.47 & 94.55±0.16 & 79.62±0.05 \\
            $[0.2,0.4]$ & 99.81±0.06 & 92.95±0.62 & 86.00±1.43 & 94.46±0.02 & 79.99±0.79 \\
            $[0.4,0.6]$ & 99.80±0.07 & 92.71±0.21 & 85.92±1.84 & 94.38±0.11 & 80.09±0.66 \\
            $[0.6,0.8]$ & 99.80±0.06 & \textbf{92.95±0.03} & \textbf{86.04±1.55} & 94.49±0.27 & \textbf{80.13±0.84} \\
            $[0.8,1.0]$ & \textbf{99.83±0.05} & 92.89±0.62 & 85.92±1.35 & \textbf{94.53±0.23} & 79.96±0.61 \\
            \bottomrule
        \end{tabular}
    }
    \caption{Performance of \ourmethod~with different $\lambda$ ranges.}
    \label{tab:lam}
\vspace{-0.3cm}
\end{table}

\noindent \textbf{Sensitivity Analysis of $\lambda$.}
In the main results reported in Table~\ref{tab:main}, $\lambda$ was randomly sampled from the interval $[0.01, 1]$ for generating mixed graphons between ID and OOD subgroups. 
Here, we conduct a finer-grained analysis by fixing $\lambda$ to specific values within this range and examining its impact on detection performance.
As shown in Table~\ref{tab:lam}, the performance sensitivity to $\lambda$ varies across different dataset pairs. However, we observe that moderate values of $\lambda$ generally lead to stronger results on most benchmarks. This suggests that a balanced interpolation effectively preserves discriminative topological patterns from both source graphons and enhances the diversity of boundary-aware samples.



\begin{figure}[ht]
\centering
\vspace{-0.2cm}
\hspace{-0.2cm}
\includegraphics[width=.87\linewidth]{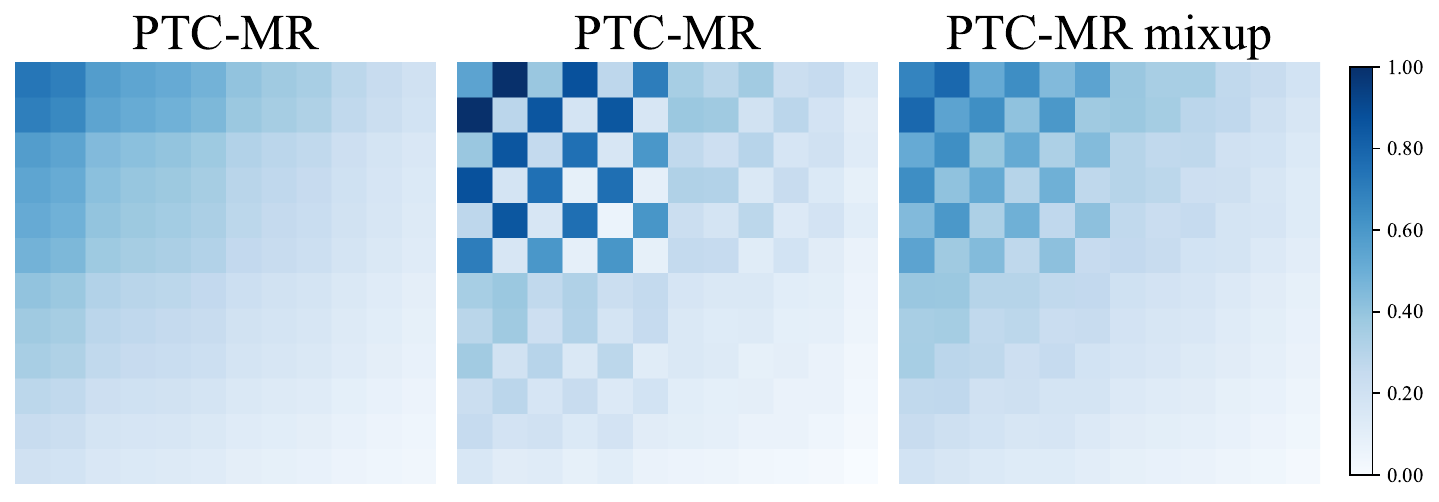}\\
\hspace{-0.2cm}
\includegraphics[width=.87\linewidth]{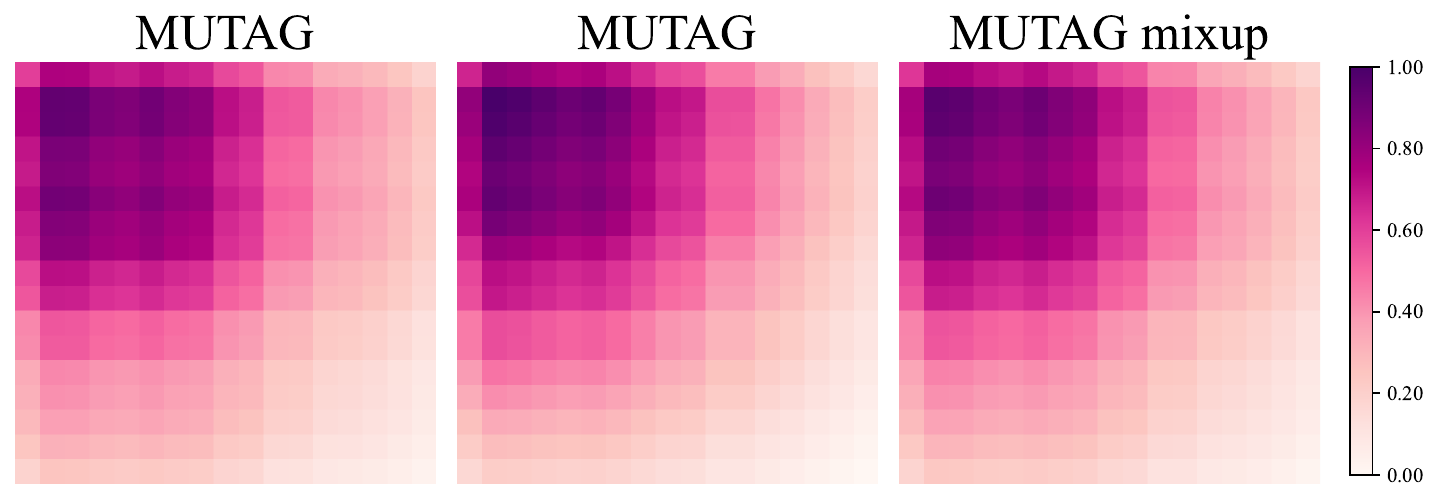}
\hspace{-0.2cm}
\vspace{-0.2cm}
\caption{Estimated graphons and their mixup results on the PTC/MUTAG (PTC as ID, MUTAG as OOD). Within each row, the first two columns are the original estimated graphons, and the third column is mixed graphon.}
\vspace{-0.3cm}
\label{fig:graphon}
\end{figure}


\begin{figure}[ht]
 \centering
\hspace{-0.18cm}
{\includegraphics[width=0.32\linewidth]{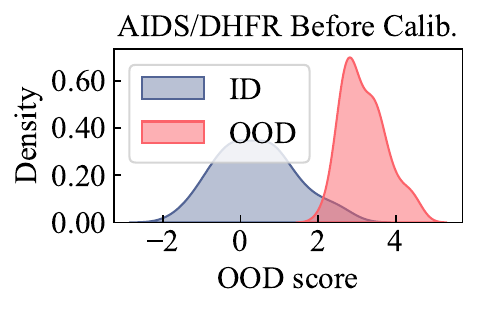}} 
\hspace{-0.18cm}
{\includegraphics[width=0.32\linewidth]{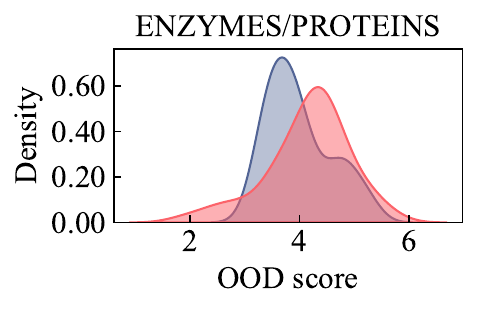}}
\hspace{-0.18cm}
{\includegraphics[width=0.32\linewidth]{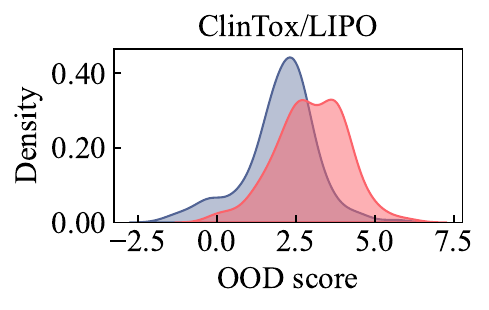}} \\

\hspace{-0.18cm}
{\includegraphics[width=0.32\linewidth]{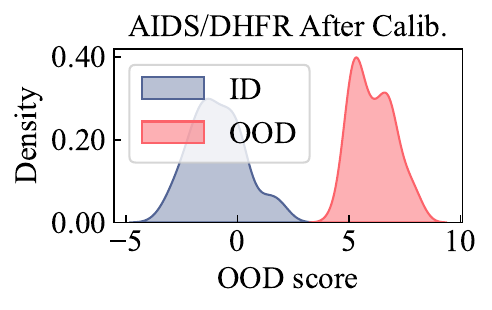}} 
\hspace{-0.18cm}
{\includegraphics[width=0.32\linewidth]{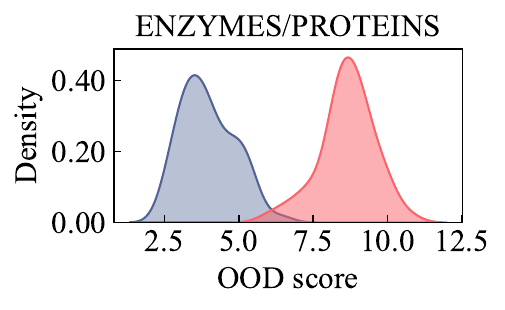}}
\hspace{-0.18cm}
{\includegraphics[width=0.32\linewidth]{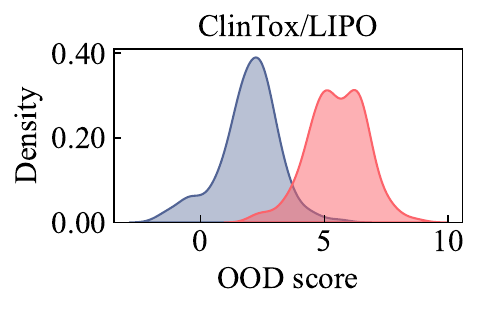}} \\
 \vspace{-0.3cm}
\caption{Score distributions on several dataset pairs. The first row shows the score distribution {\textbf{before}} calibration (abbreviated as Calib.), while the second row presents the score distribution {\textbf{after}} applying our calibration on the corresponding dataset. The overlap area between ID and OOD samples is significantly reduced after calibration using \ourmethod.}
\vspace{-0.3cm}
\label{fig:dens1}
\end{figure}

\noindent \textbf{Graphon Mixup Visualization.}
We estimate graphons of ID and OOD samples and perform graphon mixup visualized as heatmaps in Figure~\ref{fig:graphon}. 
We can observe clear structural differences between graphons from different distributions. In contrast, mixup within the same distribution preserves key structural properties while generating new graphons, effectively enhancing the diversity of discriminative typologies.
Additional visualizations are provided in Appendix E.7.

\noindent \textbf{Score Distribution Visualization.}
We visualize the OOD score distributions for ID and OOD samples across several dataset pairs before and after applying our calibration strategy in Figure~\ref{fig:dens1}. Compared to the uncalibrated setting, the overlap between ID and OOD score distributions is significantly reduced. This demonstrates that our structure-aware calibration method effectively amplifies the distributional differences between ID and OOD samples, leading to more reliable detection.

\vspace{-2.0mm}
\section{Conclusion}

In this paper, we propose \ourmethod, a boundary-aware OOD score calibration framework for test-time graph OOD detection that calibrates OOD scores without modifying pre-trained GNNs or relying on auxiliary outlier data. 
Concretely, we first perform subgroup partitioning of test samples using pre-trained scores and estimate graphons separately for ID and OOD groups. 
To address the diversity of latent structural patterns and enhance representation near the distributional boundary, we introduce a graphon mixup strategy that synthesizes discriminative topologies, which are continuously stored in dual dynamic dictionaries, maintained as priority queues. 
We then adopt a learnable attention mechanism for boundary-aware OOD score calibration, effectively reducing the overlap between ID and OOD score distributions, especially in ambiguous samples near the boundary. 
Extensive experiments across multiple benchmarks demonstrate the superiority of \ourmethod over state-of-the-art baselines.

\section*{Acknowledgments}
This work has been supported by CCSE project (CCSE-2024ZX-09).

\bibliography{reference.bib}
\appendix
\setcounter{secnumdepth}{2} 

\setcounter{table}{0}
\setcounter{figure}{0}

\setcounter{definition}{0}
\setcounter{lemma}{0}
\setcounter{theorem}{0}

\renewcommand{\thetable}{A.\arabic{table}}
\renewcommand{\thefigure}{A.\arabic{figure}}

\section{Notations}\label{sec:notation}
As an expansion of the notations in our work, we summarize the frequently used notations in Table~\ref{tab:notation}.

\begin{table}[htbp]
  \centering
  
\resizebox{.95\linewidth}{!}{
    \begin{tabular}{cl}
    \toprule
    \textbf{Notations} & \multicolumn{1}{c}{\textbf{Descriptions}} \\
    \midrule
${G} = (\mathcal{V}, \mathcal{E}, \mathbf{X})$    & Graph with the node set $\mathcal{V}$ and edge set $\mathcal{E}$ \\
$\mathcal{V}$ & The set of nodes in the graph \\ 
$\mathcal{E}$ & The set of edges in the graph \\ 
$\mathbf{X}$ & The feature matrix \\ 
$d$ & The feature dimension of the graph \\
$\mathbf{A}$ & The adjacency matrix of the graph \\
$\mathbb{P}^{id}$, $\mathbb{P}^{ood}$  & The distribution where graphs are sampled from \\
$f$ & The pre-trained GNN encoder which is frozen \\
$||\cdot||_{\Box}$ & The cut norm, and $||\cdot||_{\Box}:\mathbb{W} \rightarrow \mathbb{R}$ \\
$W, \mathbb{W}$ & The graphon \& step function and graphon space \\
$W^P$ & The step function in matrix form \\
$\lambda$ & Trade-off hyperparameter for graphon mixup\\
$S(G)$ & The calibrated OOD score of graph $G$ \\
$\mathcal{K}^{ood}, \mathcal{K}^{id}$ & The key in OOD and ID dictionary \\
$\beta$ & Trade-off hyperparameter for score calibration\\ 
    
\bottomrule
\end{tabular}%
}
\caption{Summary of notations and descriptions.}
\label{tab:notation}
\vspace{-1em}
\end{table}%

\section{Related Work}
\subsection{Graph Out-of-Distribution Detection}
Out-of-distribution detection \cite{hendrycks2016baseline,wang2022openauc} involves the task of distinguishing test samples from distributions different from the seen training data. It comprises post-hoc and fine-tuning approaches \cite{yang2021generalized}. Post-hoc methods \cite{liang2017enhancing,lee2018simple,sun2021react,wang2022watermarking} leverage the logit space and output scores of models that are trained on in-distribution data to classify ID and OOD data. Fine-tuning approaches \cite{hendrycks2018deep,du2022vos} introduce extra regularization terms during training or incorporate auxiliary training data, referred to as outlier exposure, which can be either real, synthetic, or sampled from the feature space. Outlier exposure has proven effective in enhancing OOD detection performance. 
However, these methods are typically applied to image or text data.
OOD detection for graphs remains relatively underexplored. 
These challenges highlight the need for robust representation learning methods for graphs~\cite{wu2022structural, wu2022simple,wu2023sega,wu2024uncovering,wu2025toward,hou2025structural,hou2025test}, especially in the presence of previously unseen samples~\cite{hou2024nc2d}.
Recent methods such as GLocalKD~\cite{ma2022deep} and OCGIN~\cite{zhao2023using} focus on graph anomaly detection under supervised or semi-supervised settings. GOOD-D~\cite{liu2023good} and AAGOD~\cite{guo2023data} adopt contrastive learning and graph augmentations to enhance OOD sensitivity with only ID data. 
HGOE~\cite{junwei2024hgoe} introduces a hybrid outlier exposure framework by synthesizing both internal and external graph outliers but requires access to auxiliary OOD data during training. 
More recently, GOODAT~\cite{wang2024goodat} explores a practical test-time scenario and proposes optimizing a graph masker but still requires test-time model optimization.
In this work, we propose \ourmethod, a novel OOD score calibration framework
for test-time graph OOD detection without the need for auxiliary data or pre-trained model fine-tuning.

\subsection{Test-time Training and Adaptation} 
Test-time training (TTT) aims to improve model robustness during inference by optimizing certain components using test samples. The pioneer work~\cite{sun2020test} proposes to optimize the feature extractor via an auxiliary task loss. To remove the need for training an auxiliary task, Tent~\cite{wang2020tent} minimizes the prediction entropy without accessing labeled data.
Test-time adaptation (TTA) extends this idea by adapting models at test time without modifying the training process. Recent studies in the graph domain focus on node classification tasks. For instance, GTrans~\cite{jin2022empowering} augments the test graph by generating contrastive views to enhance representation learning, while GraphPatcher~\cite{ju2023graphpatcher} generates virtual neighbors to improve low-degree node performance. These methods typically rely on heuristics or self-supervised losses to adapt to target distributions.
In contrast to TTT and TTA, which enhance model generalization under distribution shifts by fine-tuning during testing, unsupervised OOD detection serves as a prerequisite task, helping to filter unreliable test inputs before applying any adaptation or prediction. This paper focuses on test-time OOD detection, which aims to identify whether a test sample deviates from the training distribution.

\subsection{Further Discussion and Comparison with Related Methods}
Here, we discuss the connections between this paper and the most relevant studies. D2GO~\cite{hou2025test} represents the first approach that achieves test-time graph OOD detection solely through OOD score calibration. Specifically, D2GO models the essential structural patterns of ID-like and OOD-like samples by maintaining two graphon dictionaries, which respectively capture the characteristic structures of the ID and OOD domains. Inspired by this idea, our proposed method BaCa also leverages graphons to construct discriminative typologies for both ID and OOD samples.
However, after computing the similarity between the input graph and the two dictionaries, D2GO requires an additional hyperparameter to balance the contributions of the ID and OOD dictionaries when calibrating the OOD score. The impact of this hyperparameter on performance is difficult to anticipate and often depends on the specific dataset or distribution shift. To overcome this limitation, BaCa introduces a learnable attention mechanism for boundary-aware OOD score calibration. This mechanism adaptively adjusts the relative influence of the ID and OOD graphons without the need for manual hyperparameter tuning, while adding negligible computational overhead. As a result, BaCa achieves stable and competitive performance across different graph OOD detection settings.

In addition to graph OOD detection, it is worth noting that anomaly detection~\cite{wang2021cost,wang2022search,wang2024cost,zhu2023meter,zhu2024llms} represents a closely related yet distinct setting. While both tasks aim to identify samples that deviate from known distributions, anomaly detection typically focuses on detecting rare or abnormal instances within the same domain, rather than distributional shifts between domains. Nevertheless, the underlying principle of identifying boundary or low-density regions in representation space~\cite{fang2025omega,fang2024information,fang2023distributed} is shared across the two problems. In this sense, the boundary-aware calibration strategy developed in BaCa has the potential to enhance anomaly detection as well, by providing more fine-grained control over the decision boundary between normal and abnormal samples. Exploring how the proposed framework can be adapted to general anomaly detection scenarios is a promising direction for future work.

\section{Algorithm}\label{sec:algorithm}

\IncMargin{1.3em}
\begin{algorithm}[t!]
    \caption{Test-time optimization process of \ourmethod.}
    \label{alg:baca}
    \KwIn{Test sample $G$; Pre-trained GNN encoder $f$ (frozen); Number of test-time iterations $T$; Hyperparameters $\lambda$, $\mathbb{K}$, $\beta$; Initial ID/OOD dictionaries $\mathcal{K}^{id}$ and $\mathcal{K}^{ood}$.}
    \KwOut{Calibrated OOD score $S_{\text{BaCa}}(G)$.}

    Compute pre-trained score $S_{\text{Pre}}(G) \gets$ Eq. (1)\;
    \tcp{Boundary-Aware Latent Pattern Modeling}
    Estimate graphons for each subgroup using step function approximation\;
    Perform subgroup partitioning $\mathcal{D}_{\text{test}} = \mathcal{C}^{\text{id}} \cup \mathcal{C}^{\text{ood}}$\;
    \tcp{Graphon Mixup for Discriminative Typology Expansion}
    \For{each pair $(W_i, W_j)$ within ID or OOD subgroup}{
        Generate mixed graphon $W_s = \lambda W_i + (1 - \lambda) W_j$ using Eq. (3)\;
        Sample synthetic graph $\tilde{G}$ from $W_s$ using Eq. (2)\;
        Update ID/OOD dictionary using score-based priority queues\;
    }

    \tcp{Adaptive Calibration via Dual Dynamic Dictionary}
    \For{$t = 1$ to $T$}{
        Obtain representation $q = f(G)$\;
        \tcp{Attention-based Score via Top-$\mathbb{K}$ Dictionary Entries}
        Select top-$\mathbb{K}$ keys from ID/OOD dictionaries by cosine similarity\;
        Compute attention outputs $\text{ATTN}_{\text{in}}, \text{ATTN}_{\text{out}}$ using Eq. (6) and~(7)\;
        Compute calibrated score $S_{\text{Attn}}(G) = S_{\text{in}} + S_{\text{out}} \gets$ Eq.~(8)\;
        Compute final score $S_{\text{BaCa}}(G) = S_{\text{Pre}}(G) + \beta \cdot S_{\text{Attn}}(G) \gets$ Eq.~(9) \;
        Calculate attention loss $\mathcal{L} \gets$ using Eq.~(10)\;
        Update attention parameters $(\mathbf{W}_Q, \mathbf{W}_K, \mathbf{W}_V)$ via gradient descent\;
    }

    \Return $S_{\text{BaCa}}(G)$
\end{algorithm}
\DecMargin{1.3em}

The overall optimization process of our \ourmethod~is shown in Algorithm~\ref{alg:baca}.

\section{Theorem Proofs} \label{ap:thm}
In this section, we provide the detailed proof of Theorem 1. We first introduce the necessary preliminaries, and then present the complete proof in Section~\ref{ap:proof-thm1}.

\subsection{Preliminaries} \label{ap:proof-pre}

\begin{definition}[{Graphon}]
A graphon is a symmetric, two-dimensional, and continuous measurable function $W:\Omega^2\rightarrow [0,1]$, where $\Omega$ is a measurable space, \eg $\Omega=[0,1]$. Here symmetric means $W(x,y)=W(y,x)$ for all $x,y\in \Omega$. 
\end{definition}
Given a certain generation rule of graphs, graphon is considered to continuously add nodes according to this rule until infinity, and finally obtain a probability density function. It describes the probability that an edge exists between two nodes. Specifically, given two nodes $v_i, v_j \in [0,1]$, $W(i,j)$ describes the probability that an edge exists between these nodes. 
Indeed, to measure the distance between graphons, it's essential to introduce a suitable distance function.
We first begin with the cut norm \cite{lovasz2012large}.

\begin{definition}[{Cut Norm}]
The cut norm of graphon $W$ is defined as:
\begin{equation}
\|W\|_{\square} = \sup_{S, T\subseteq \Omega} \left| \int_{S \times T} W(x, y) \, dxdy \right|,
\end{equation}
where the supremum is taken over all subsets $S, T \subseteq [0,1]$.
\label{def:cut-norm}
\end{definition}

\begin{definition}[{Homomorphism density}]
Given a graph $F = (\tilde{\mathcal{V}}, \tilde{\mathcal{E}})$ and graphon $W$, the homomorphism density of $F$ in $W$ is:
\begin{equation}
t(F, W) = \int_{[0,1]^{|\tilde{\mathcal{V}}|}} \prod_{(i, j) \in \tilde{\mathcal{E}}} W(x_i, x_j) \prod_{i \in \tilde{\mathcal{V}}} dx_i.
\end{equation}
\end{definition}
\noindent We now present a standard result bounding the change in homomorphism density under graphon perturbation.

\begin{lemma}[{Counting Lemma}~\cite{lovasz2012large}]\label{lem:graphon_diff}
Let $F$ be a simple graph and $W, W'$ be graphons. Then
\begin{equation}
|t(F, W) - t(F, W')| \leq \mathrm{e}(F) \cdot \| W - W' \|_{\square},
\end{equation}
where $\mathrm{e}(F)$ is the number of edges in $F$.
\end{lemma}

\subsection{Proof of Theorem 1} \label{ap:proof-thm1}
To begin with, we revisit the theorem as follows.

\begin{theorem}
Let $W_G$ and $W_H$ be graphons estimated from two subgroups $G$ and $H$ of the same distribution type (\ie, both ID or both OOD). Let the interpolated graphon be defined as $W_s = \lambda W_G + (1 - \lambda) W_H$, where $\lambda \in [0,1]$. Then, for any discriminative typology $T_G$ and $T_H$:
\begin{equation}
\begin{aligned}
    \left| t(T_G, W_s) - t(T_G, W_G) \right| &\leq (1 - \lambda) \cdot \delta_{GH}, \\
    \left| t(T_H, W_s) - t(T_H, W_H) \right| &\leq \lambda \cdot \delta_{GH},
\end{aligned}
\end{equation}
where $\delta_{GH} = \| W_G - W_H \|_{\square}$ is the cut norm distance between $W_G$ and $W_H$.
\end{theorem}

\begin{proof}
Let $W_G$ and $W_H$ be the graphons estimated from two subgroups $G$ and $H$ belonging to the same distribution type (ID or OOD). Let the mixed graphon be defined as:
\begin{equation}
W_s = \lambda W_G + (1 - \lambda) W_H, \quad \lambda \in [0,1].
\end{equation}

We aim to bound the deviation in homomorphism density of a discriminative typology $T_G$ (from group $G$) under interpolation.
Applying Lemma~\ref{lem:graphon_diff} with $F = T_G$, $W = W_s$, and $W' = W_G$, we have:
\begin{equation}
\begin{aligned}
&\left| t(T_G, W_s) - t(T_G, W_G) \right| \\
= &\left| t\left(T_G, \lambda W_G + (1 - \lambda) W_H \right) - t(T_G, W_G) \right| \\
\leq &\mathrm{e}(T_G) \cdot \left\| (1 - \lambda)(W_H - W_G) \right\|_{\square} \\
= &(1 - \lambda) \cdot \mathrm{e}(T_G) \cdot \| W_H - W_G \|_{\square}.
\end{aligned}
\end{equation}

Similarly, for $T_H$ from group $H$, we have:
\begin{equation}
\begin{aligned}
&\left| t(T_H, W_s) - t(T_H, W_H) \right| \\
= &\left| t\left(T_H, \lambda W_G + (1 - \lambda) W_H \right) - t(T_H, W_H) \right| \\
\leq &\lambda \cdot \mathrm{e}(T_H) \cdot \| W_H - W_G \|_{\square}.
\end{aligned}
\end{equation}

By absorbing $\mathrm{e}(T_G)$ and $\mathrm{e}(T_H)$ into the definition of $T$ or considering constant-size motifs, the result simplifies to:
\begin{equation}
\begin{aligned}
\left| t(T_G, W_s) - t(T_G, W_G) \right| &\leq (1 - \lambda) \cdot \delta_{GH}, \\
\left| t(T_H, W_s) - t(T_H, W_H) \right| &\leq \lambda \cdot \delta_{GH},
\end{aligned}
\end{equation}
where $\delta_{GH} = \| W_H - W_G \|_{\square}$.

\end{proof}


\section{Experiment}
\subsection{Dataset Description} \label{app:data}

For OOD detection, we employ 10 pairs of datasets from two mainstream graph data benchmarks (i.e., TUDataset~\cite{tu_Morris2020} and OGB~\cite{ogb_hu2020open}) following GOOD-D~\cite{liu2023good}. Specifically, we select 8 pairs of molecular datasets, 1 pair of protein datasets, and 1 pair of social network datasets.
$90\%$ of ID samples are used for training, and $10\%$ of ID samples and the same number of OOD samples are integrated together for testing. 
The partitioning of ID samples for training, along with the division of ID and OOD samples for testing, follows GOOD-D~\cite{liu2023good}. 
Further detailed information about these datasets is categorized and described as follows.

\subsubsection{Molecular Datasets}

\begin{itemize}[leftmargin=1.5em]
    \item \textbf{BZR}~\cite{tu_Morris2020} is a dataset focused on benzodiazepine receptor ligands, containing molecular structures and associated binding affinities. It is crucial for drug design and discovery, specifically for studying receptor-ligand interactions.

    \item \textbf{PTC-MR}~\cite{tu_Morris2020} reports the carcinogenicity of 344 chemical compounds in male and female rats and includes 19 discrete labels. It is utilized for predicting the carcinogenic potential of chemical substances.

    \item \textbf{AIDS}~\cite{tu_Morris2020} contains data on anti-HIV compounds, including their molecular structures and biological activities, serving as a valuable resource for the development of anti-HIV drugs.

    \item \textbf{ENZYMES}~\cite{tu_Morris2020} is a dataset consisting of protein structures classified into enzyme types based on their functionality. It is used for protein function prediction and enzyme classification.

    \item \textbf{COX2}~\cite{tu_Morris2020} comprises data on cyclooxygenase-2 inhibitors, which are compounds with anti-inflammatory properties. This dataset is essential for research and development of anti-inflammatory drugs.

    \item \textbf{MUTAG}~\cite{tu_Morris2020} has seven kinds of graphs derived from 188 mutagenic aromatic and heteroaromatic nitro compounds. It is used for studying the mutagenicity of chemical substances.

    \item \textbf{DHFR}~\cite{tu_Morris2020} includes dihydrofolate reductase inhibitors, important in the development of antibacterial and anticancer drugs, aiding in drug discovery and medicinal chemistry research.

    \item \textbf{PROTEINS}~\cite{tu_Morris2020} contains data on protein structures and their functionalities. Nodes represent secondary structure elements (SSEs), and edges connect neighboring elements in the amino acid sequence or 3D space. This dataset is used for protein structure prediction and functional analysis.

    \item \textbf{Tox21}~\cite{ogb_hu2020open} is a dataset containing toxicity data on 12 biological targets, which has been used in the 2014 Tox21 Data Challenge and includes nuclear receptors and stress response pathways.

    \item \textbf{BBBP}~\cite{ogb_hu2020open, martins2012bayesian} includes records of whether a compound has the permeability property of penetrating the blood-brain barrier, essential for the design of central nervous system drugs.

    \item \textbf{ClinTox}~\cite{ogb_hu2020open, novick2013sweetlead, gayvert2016data} contains clinical toxicity data on a variety of drug compounds, classifying drugs approved by the FDA and those that have failed clinical trials for toxicity reasons.

    \item \textbf{ToxCast}~\cite{ogb_hu2020open, richard2016toxcast} includes high-throughput screening data on the toxicity of chemical substances, with measurements based on over 600 in vitro screenings. This dataset is used for large-scale toxicity assessment and environmental health research.

    \item \textbf{SIDER}~\cite{ogb_hu2020open, kuhn2016sider} contains information on drug side effects, grouped into 27 system organ classes, also known as the Side Effect Resource. It is utilized for predicting drug side effects and improving drug safety profiles.

    \item \textbf{BACE}~\cite{ogb_hu2020open, subramanian2016computational} includes qualitative binding results for a set of inhibitors of human $\beta$-secretase 1, which are potential treatments for Alzheimer's disease. This dataset is used in Alzheimer's disease research and drug development.

    \item \textbf{FreeSolv}~\cite{ogb_hu2020open} includes data on the hydration free energy of small molecules, used for molecular dynamics simulations and solubility studies.

    \item \textbf{Esol}~\cite{ogb_hu2020open} contains data on the aqueous solubility of compounds, used for studying compound solubility and drug design.

    \item \textbf{LIPO}~\cite{ogb_hu2020open} includes data on the lipophilicity of chemical compounds. It is used for studying the partitioning of compounds between water and oil phases, which is important in drug design.

    \item \textbf{MUV}~\cite{ogb_hu2020open, gardiner2011effectiveness} includes data on the activity of compounds from virtual screening, designed for validation of virtual screening techniques.

    \item \textbf{HIV}~\cite{ogb_hu2020open} contains experimentally measured abilities to inhibit HIV replication.
\end{itemize}

\subsubsection{Protein Datasets}

\begin{itemize}[leftmargin=1.5em]
    \item \textbf{PROTEINS}~\cite{tu_Morris2020} contains data on protein structures and their functionalities. Nodes represent secondary structure elements (SSEs), and edges connect neighboring elements in the amino acid sequence or 3D space. This dataset is used for protein structure prediction and functional analysis.

    \item \textbf{ENZYMES}~\cite{tu_Morris2020} is a dataset consisting of protein structures classified into enzyme types based on their functionality. It is used for protein function prediction and enzyme classification.
\end{itemize}

\subsubsection{Social Network Datasets}

\begin{itemize}[leftmargin=1.5em]
    \item \textbf{IMDB-BINARY}~\cite{tu_Morris2020} (abbreviated as IMDB-B) is derived from the collaboration of a movie set. Each graph consists of actors or actresses, with edges representing their cooperation in a movie. The label corresponds to movie's genre. This dataset is used for movie classification and recommendation system studies.

    \item \textbf{IMDB-MULTI}~\cite{tu_Morris2020} (abbreviated as IMDB-M) consists of graphs derived from movie collaborations which is similar to IMDB-BINARY, but with multi-class labels. It is utilized in multi-class classification tasks in social network analysis.
\end{itemize}

\subsection{Baselines} \label{app:baseline}
We compare \ourmethod with a wide range of graph OOD detection baselines, grouped into the following categories:

\begin{itemize}[leftmargin=1.5em]
    \item \textbf{Graph Kernel Based Methods.}
    These methods first extract representations using hand-crafted kernels and then apply OOD detectors. We adopt Weisfeiler-Lehman (WL)~\cite{wlgk_shervashidze2011weisfeiler} and propagation kernel (PK)~\cite{pk_neumann2016propagation}, followed by local outlier factor (LOF)~\cite{lof_breunig2000lof}, one-class SVM (OCSVM)~\cite{ocsvm_manevitz2001one}, and isolation forest (iF)~\cite{iforest_liu2008isolation}.

    \item \textbf{Anomaly Detection Methods.}
    These methods jointly optimize the encoder and detection objective in a fully integrated framework. We include OCGIN~\cite{ocgin_zhao2021using}, which trains a GIN encoder via an SVDD objective; GLocalKD~\cite{glocalkd_ma2022deep}, which performs local-global distillation; and GOOD-D~\cite{liu2023good} as a contrastive learning-based end-to-end OOD detector. We also include HGOE~\cite{junwei2024hgoe}, which synthesizes internal and external outliers through outlier exposure, but requires auxiliary OOD data during training.

    \item \textbf{Self-supervised Training Methods.}
    These methods utilize self-supervised learning to obtain graph-level embeddings, then apply separate OOD detectors. We consider InfoGraph~\cite{infograph_sun2020infograph} and GraphCL~\cite{graphcl_you2020graph} as representation learners, and use iF~\cite{iforest_liu2008isolation} and Mahalanobis distance (MD)~\cite{ssd_sehwag2020ssd} for detection. We also include GOOD-D~\cite{liu2023good}, a strong baseline that integrates graph contrastive learning and perturbation-free augmentation.
    HGOE~\cite{junwei2024hgoe} introduces a hybrid outlier exposure framework by synthesizing both internal and external graph outliers but requires access to auxiliary OOD data during training. 

    \item \textbf{Test-time and Data-centric Methods.}
    These methods perform OOD detection during inference without modifying the pretrained GNN. We include AAGOD~\cite{guo2023data}, which adopts contrastive learning and graph augmentations to enhance OOD sensitivity with only ID data, and GOODAT~\cite{wang2024goodat}, which partitions test graphs and trains a graph masker online without tuning the backbone encoder. Both serve as strong test-time baselines under practical constraints.
\end{itemize}

\subsection{Pre-trained Models}
We adopt the InfoNCE loss as the pretraining objective $\mathcal{L}_{\text{Pre}}$. Specifically, we follow GOOD-D~\cite{liu2023good} and use a 5-layer GIN~\cite{gin_xu2019how} as the encoder backbone. All models are pre-trained only on ID training data without any access to OOD samples or auxiliary supervision.
Our proposed \ourmethod is instantiated on top of the well-trained encoder, and improves OOD detection performance in a fully post-hoc and test-time setting without auxiliary data.

\begin{table*}[t]
\centering

\resizebox{1\textwidth}{!}{
\begin{tabular}{cc | cccccccccc}
\toprule
\multirow{2}{*}{\textit{ID Dict.}} & \multirow{2}{*}{\textit{OOD Dict.}} & BZR & PTC-MR & AIDS & ENZYMES & IMDB-M & Tox21 & FreeSolv & BBBP & ClinTox & Esol \\
&   & COX2 & MUTAG & DHFR & PROTEIN & IMDB-B & SIDER & ToxCast & BACE & LIPO & MUV \\
\midrule
\xmark & \xmark & 92.95±0.15& 77.59±4.37 & 99.24±0.06 & 63.14±0.00 &75.14±1.81 & 65.07±1.32 & 77.62±1.31 & 86.51±0.74 & 76.92±1.47 & 77.62±1.31 \\

\xmark & \cmark & 93.22±0.12 & \underline{85.71±1.88} & \underline{99.80±0.04} & 65.51±2.39 & 78.50±1.34 & 67.05±1.35 & 74.84±0.24 & 89.64±0.51 & 78.60±0.39 & 92.51±0.38 \\

\cmark & \xmark  & \underline{93.66±0.03} & 84.65±2.45 & 99.50±0.01 & \underline{66.30±2.43} & \underline{79.83±1.23} & \underline{68.29±1.54} & \underline{75.60±0.31} & \underline{91.49±0.41} & \underline{80.02±0.88} & \underline{94.54±0.03} \\

\midrule
\cmark & \cmark
& \textbf{94.23±0.42} & \textbf{86.53±1.39} & \textbf{99.86±0.03} & \textbf{67.10±1.43} & \textbf{80.93±0.69}& \textbf{69.82±0.59} & \textbf{83.12±0.42} & \textbf{93.11±0.29} & \textbf{82.57±0.23} & \textbf{95.31±0.14} \\
\bottomrule

\end{tabular}}
\caption{Ablation study results of \ourmethod and its variants in terms of AUC ($\%$, mean $\pm$ std). The best and runner-up results are highlighted with \textbf{bold} and \underline{underline}, respectively.} 
\label{tab:ablation}
\vspace{-0.3cm}
\end{table*}

\begin{figure*}[t]
  \centering
  \begin{minipage}[b]{0.48\textwidth}
    \centering
\includegraphics[width=\linewidth]{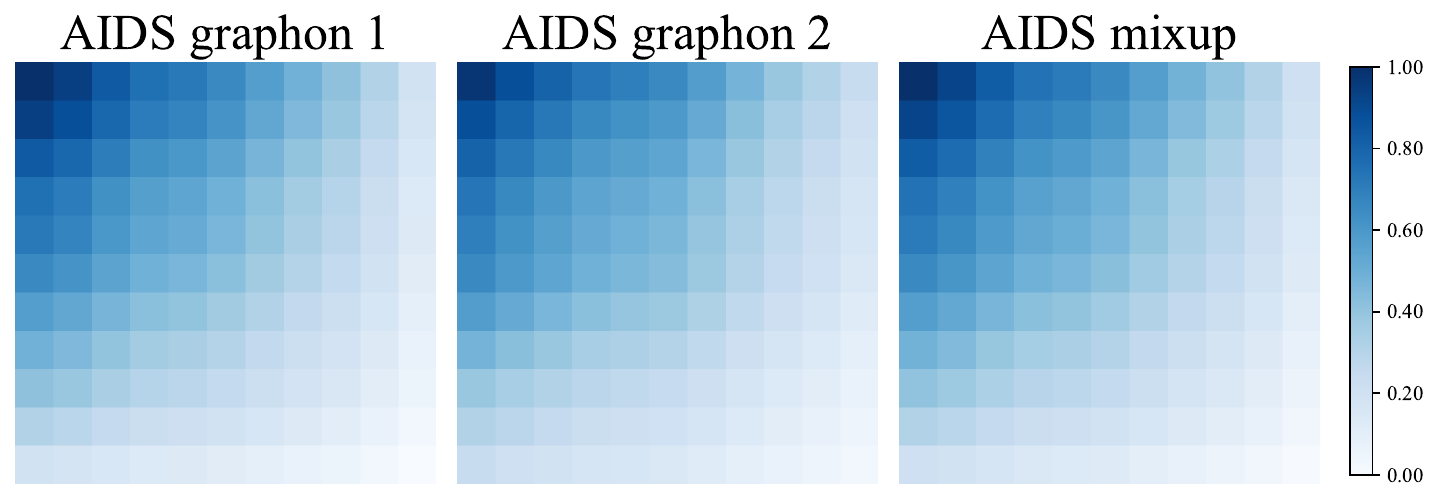} \\
\includegraphics[width=\linewidth]{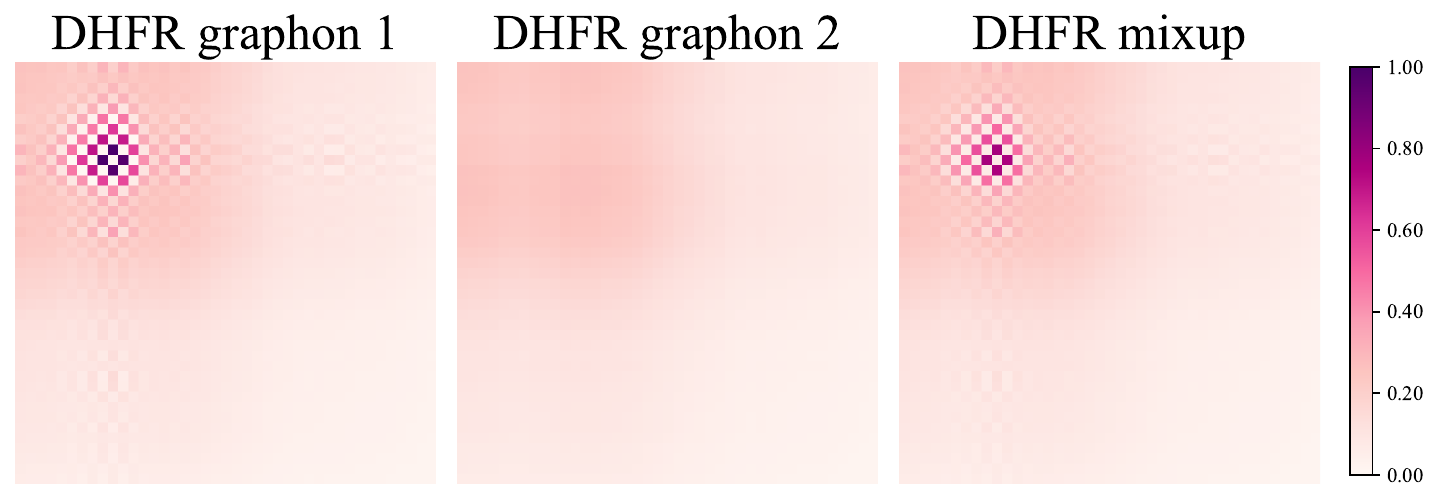}
\caption{Estimated graphons and their mixup results on the AIDS/DHFR.}
\label{fig:on-1}
  \end{minipage}
  \hfill
  \begin{minipage}[b]{0.48\textwidth}
    \centering
\includegraphics[width=\linewidth]{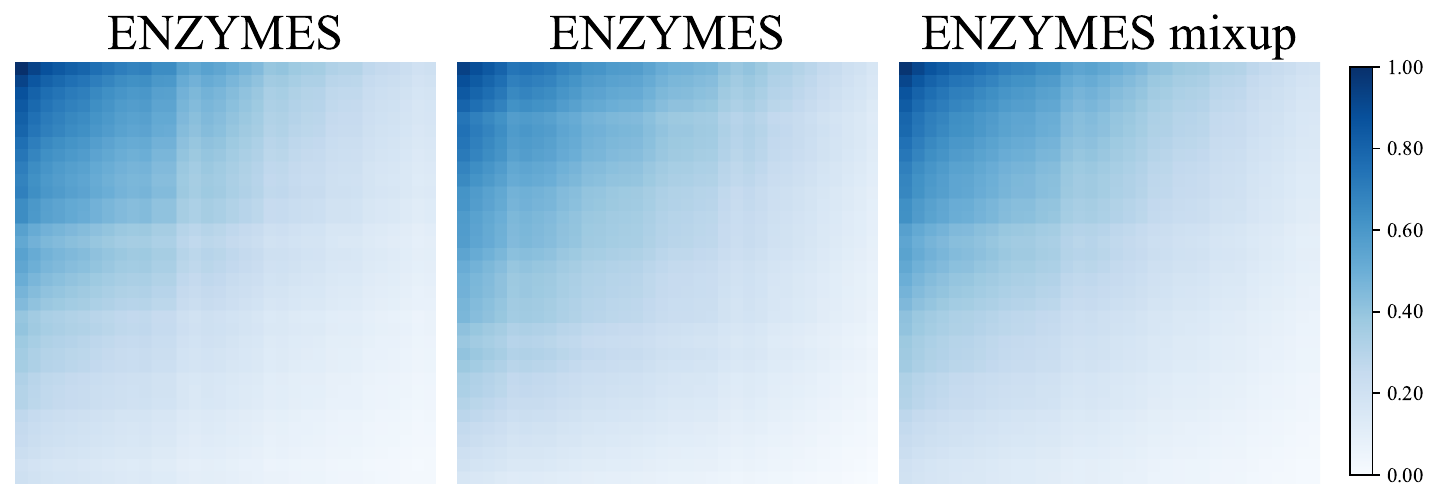} \\
\includegraphics[width=\linewidth]{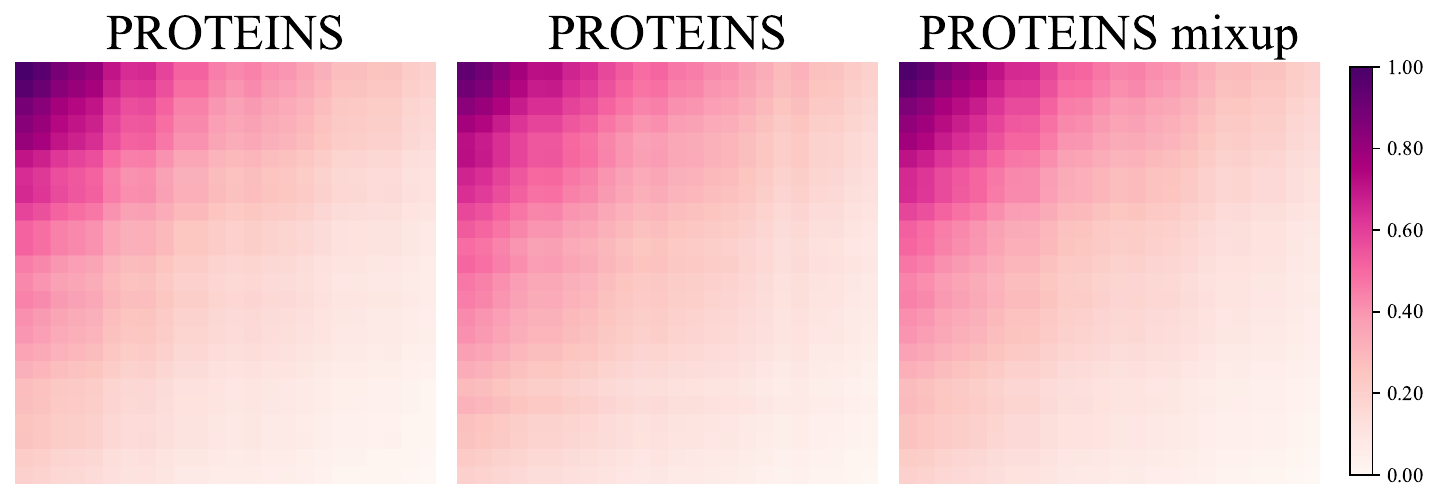}
\caption{Estimated graphons and their mixup results on the ENZYMES/PROTEIN.}
  \end{minipage}
\vspace{-0.3cm}
\end{figure*}

\subsection{Implementation}
We evaluate \ourmethod with a popular OOD detection metric, i.e., area under receiver operating characteristic Curve (AUC). Higher AUC values indicate better performance.  
The reported results are the mean performance with standard deviation after 5 runs.
We perform grid search to select the key hyper-parameters of \ourmethod.
We fixed $\mathbb{K}=5$ for all our experiments.
During ID and OOD graphon mixup, \(\lambda\) was randomly chosen from the range \([0.01,1.0]\). 

For the experiment environment, we use the following software framework: Python 3.7, Pytorch 1.8, CUDA 11.0, and Pytorch-Geometric 2.0.1. The hardware setup includes an Intel(R) Xeon(R) Gold 6240 CPU @ 2.60GHz, 256GB RAM, and a Tesla V100 PCIe 32GB GPU.

\begin{figure}[ht]
 \centering
 \hspace{-0.18cm}
 \subfigure[IMDB-M]{
   \includegraphics[width=0.23\linewidth]{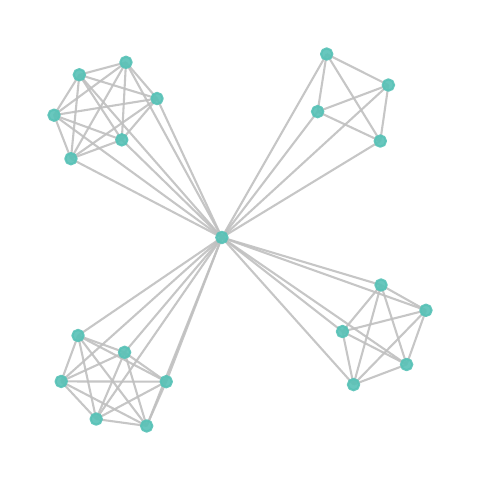}
   \includegraphics[width=0.23\linewidth]{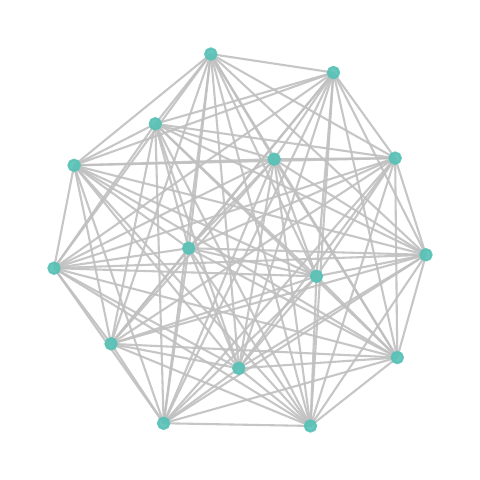}
   \label{subfig:1}}
 \hspace{-0.18cm}
 \subfigure[IMDB-B]{
   \includegraphics[width=0.23\linewidth]{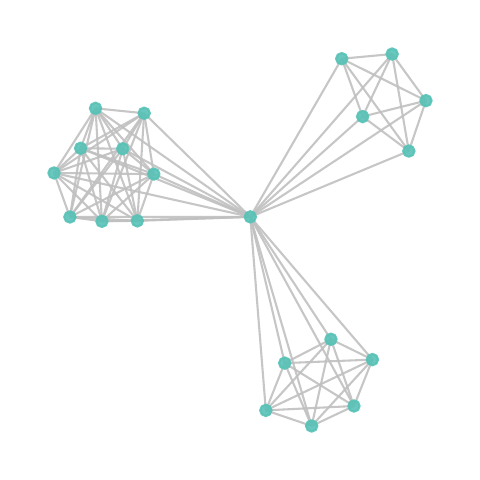}
   \includegraphics[width=0.23\linewidth]{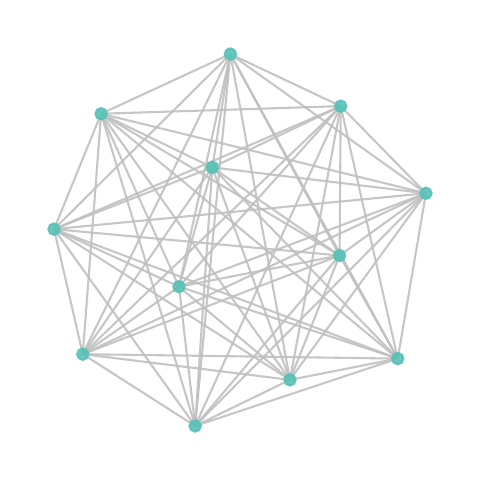}
   \label{subfig:2}} 
   
 \vspace{-0.2cm}
  \caption{Visualization on IMDB-M and IMDB-B.}
  \vspace{-0.3cm}
 \label{fig:case}
\end{figure}

\begin{figure*}[t]
  \centering
  \begin{minipage}[b]{0.48\textwidth}
    \centering
\includegraphics[width=\linewidth]{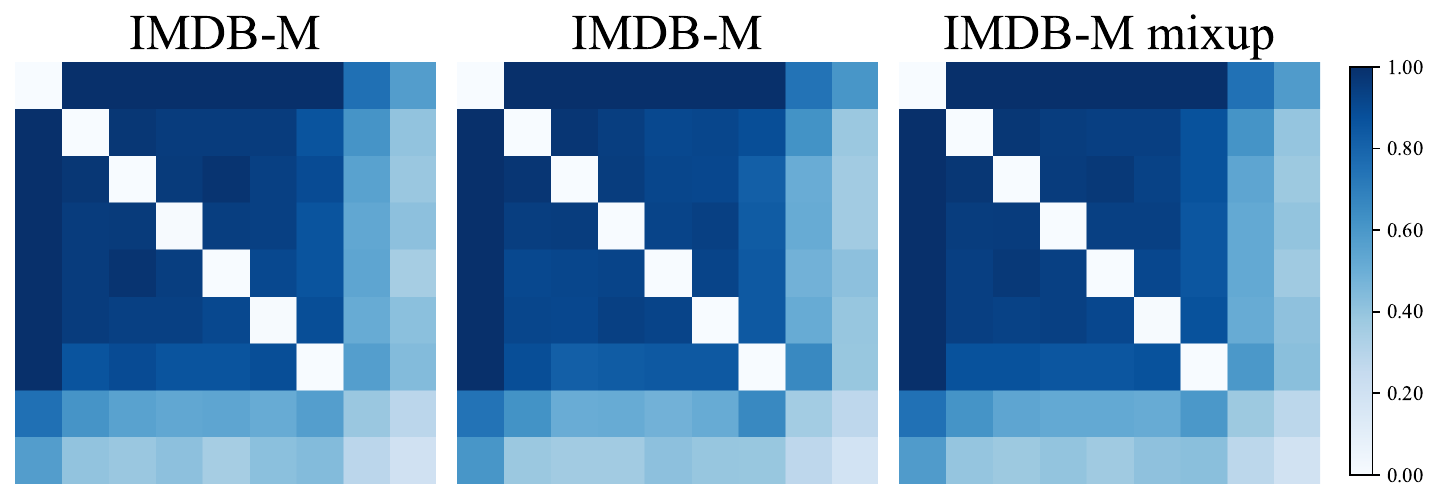} \\
\includegraphics[width=\linewidth]{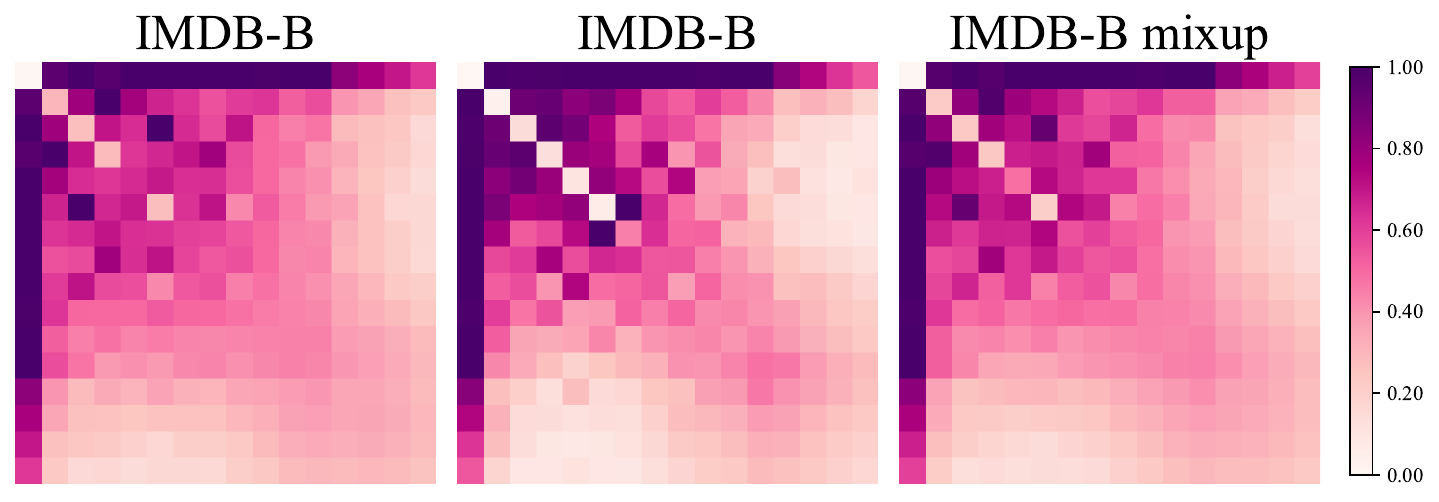}
\caption{Estimated graphons and their mixup results on the IMDB-M/IMDB-B.}
  \end{minipage}
  \hfill
  \begin{minipage}[b]{0.48\textwidth}
    \centering
\includegraphics[width=\linewidth]{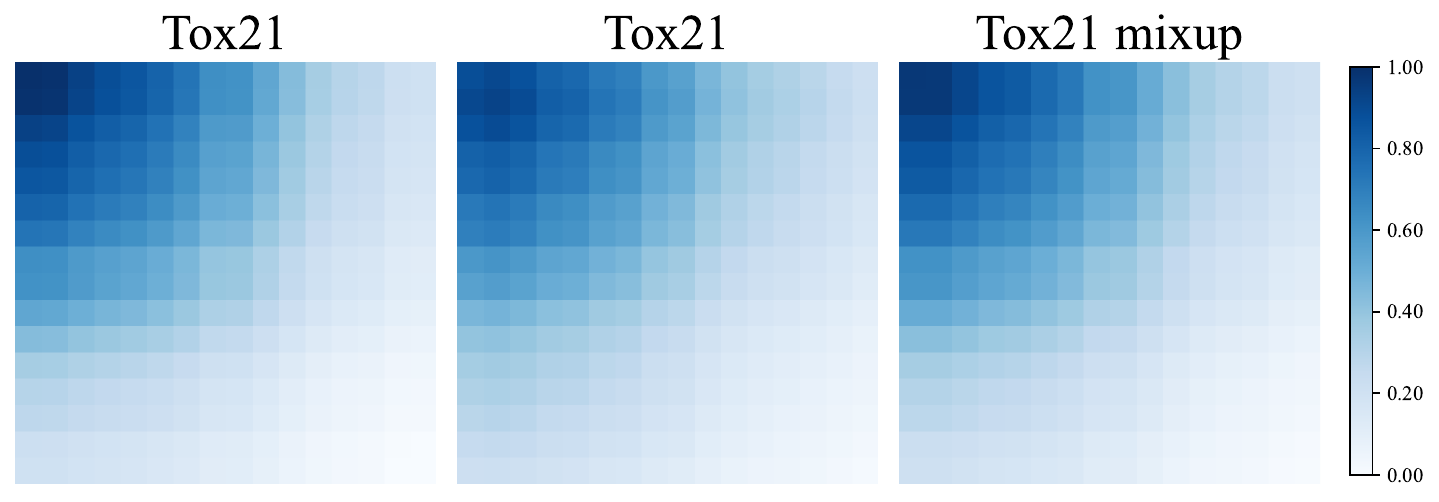} \\
\includegraphics[width=\linewidth]{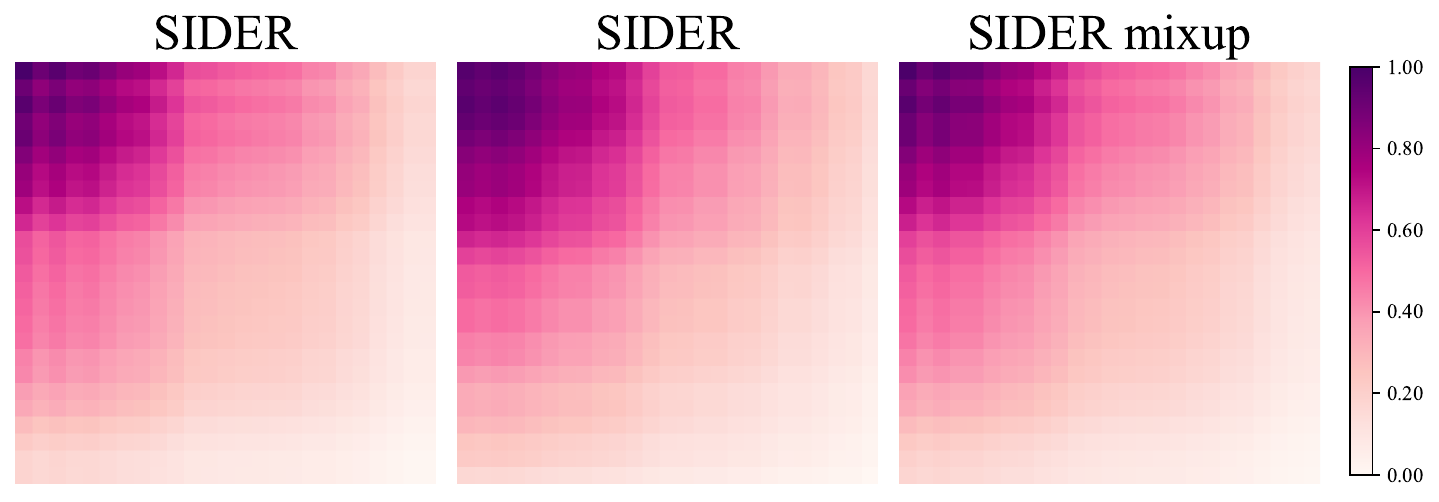}
\caption{Estimated graphons and their mixup results on the Tox21/SIDER.}
  \end{minipage}
\end{figure*}

\subsection{Case Study on IMDB-B/IMDB-M datasets}    \label{app:case}

To further elucidate the phenomenon of \ourmethod's suboptimal results on test graphs from social networks, we provide visualizations in Figure~\ref{fig:case} on IMDB-B and IMDB-M datasets. The two datasets are distinguished solely by their ground-truth labels, binary versus multi-class classification, while both originate from the same data source and thus exhibit similar structural information (\eg, star-shaped and mesh-like structures). 
Consequently, their inherent semantic information within the structure is also similar.  
Especially for the test-time setting, when only the given test samples are available, the performance of constructing ID and OOD dictionaries based on structurally similar samples is limited, making it challenging to differentiate by capturing discriminative topology.
A potential solution is to introduce test-time learnable components or feature-aware mixup, which is a promising direction for future work.

\subsection{Additional Results of Ablation Study}

We perform ablation studies by selectively removing the ID dictionary and OOD dictionary (denoted as \textit{ID Dict.} and \textit{OOD Dict.}, respectively) in Table~\ref{tab:ablation}.
We first observe that \ourmethod with both dictionaries (last row) consistently achieves the best performance across all dataset pairs, highlighting the effectiveness of our dual-dictionary design. The first row corresponds to removing both dictionaries, which reduces the model to the pretrained baseline without score calibration. Notably, using only one of the dictionaries (either ID or OOD) leads to a clear drop in performance, indicating that both are necessary to enable boundary-aware score calibration.

\subsection{Additional Graphon mixup Visualization}    \label{app:add-mix}
After partitioning the test data into ID and OOD subsets, we estimate graphons within each group and perform graphon mixup. 
In this section, we provide additional visualizations of the graphon mixup results on multiple dataset pairs, visualized as heatmaps from Figure~\ref{fig:on-1} to Figure~\ref{fig:on-3}.
The first row corresponds to the ID distribution, and the second row to the OOD distribution. In each row, the first two columns show the original estimated graphons, and the third column shows the mixed graphon computed from the first two.
We can observe clear structural differences between graphons from different distributions. In contrast, mixup within the same distribution preserves key topology while generating new graphons, effectively enhancing the diversity of discriminative typologies.

\begin{figure*}[t]
  \centering
  \begin{minipage}[b]{0.48\textwidth}
    \centering
\includegraphics[width=\linewidth]{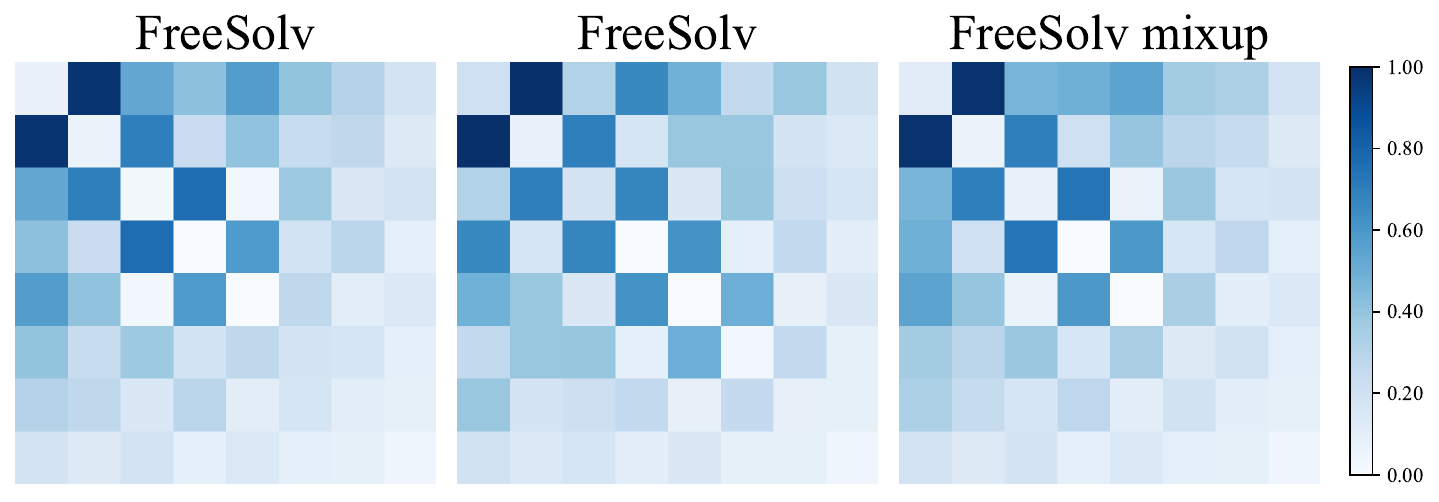} \\
\includegraphics[width=\linewidth]{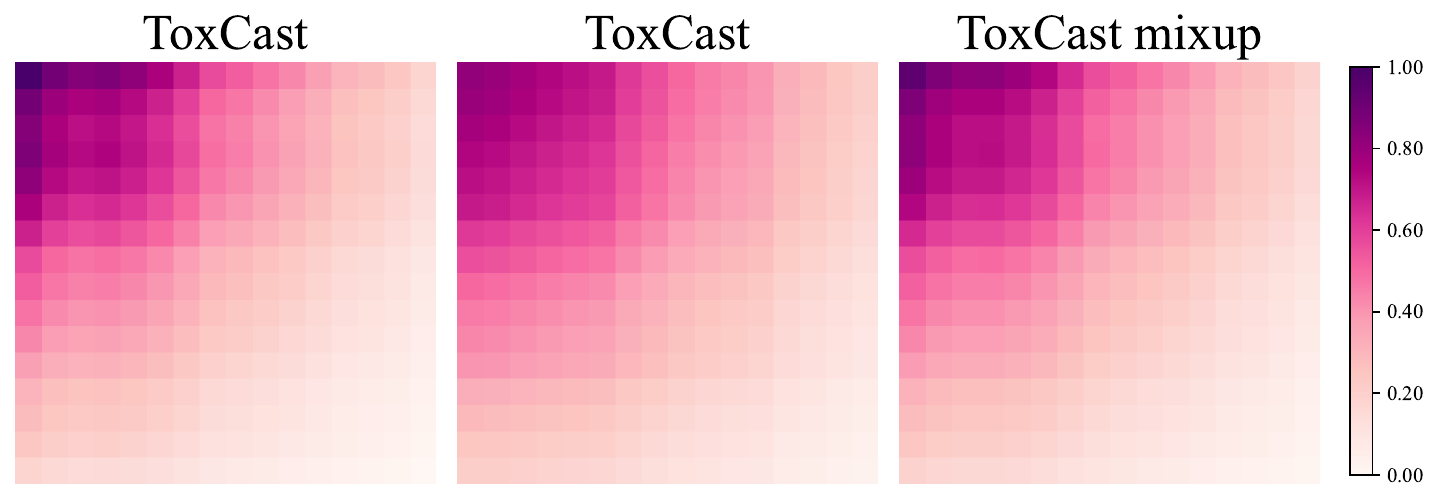}
\caption{Estimated graphons and their mixup results on the FreeSolv/ToxCast.}
  \end{minipage}
  \hfill
  \begin{minipage}[b]{0.48\textwidth}
    \centering
\includegraphics[width=\linewidth]{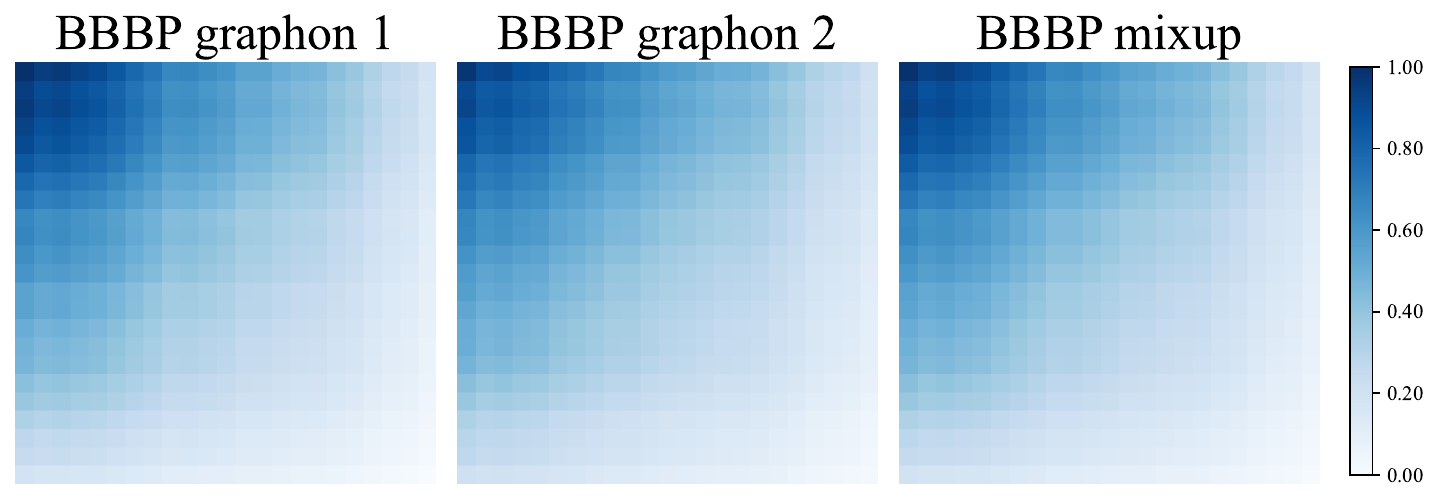} \\
\includegraphics[width=\linewidth]{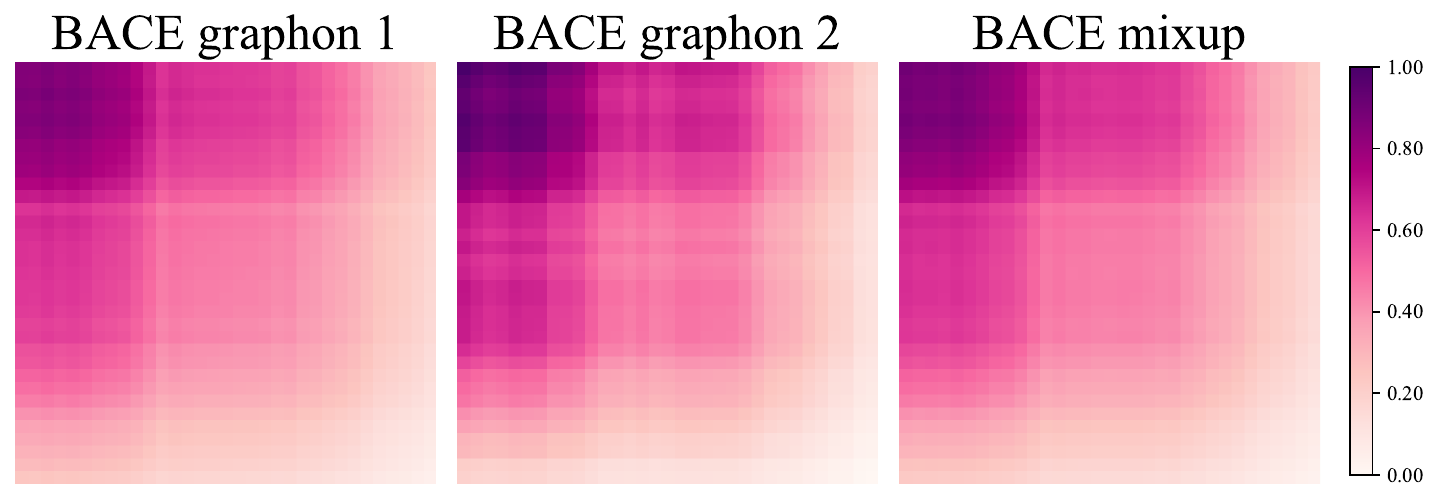}
\caption{Estimated graphons and their mixup results on the BBBP/BACE.}
  \end{minipage}
\end{figure*}

\begin{figure*}[t]
  \centering
  \begin{minipage}[b]{0.48\textwidth}
    \centering
\includegraphics[width=\linewidth]{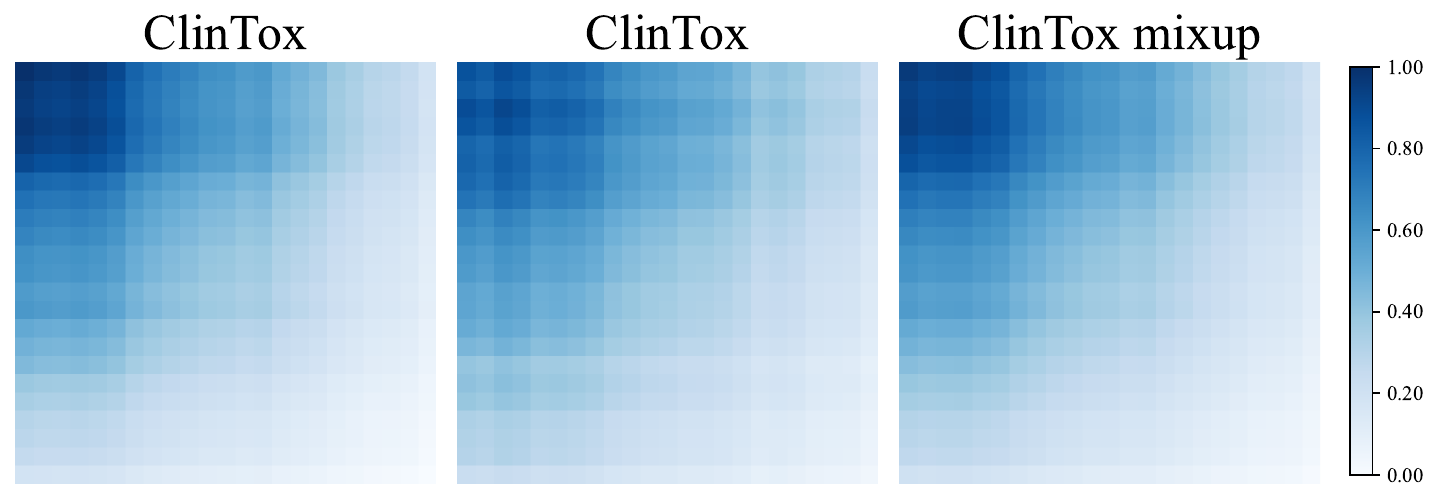} \\
\includegraphics[width=\linewidth]{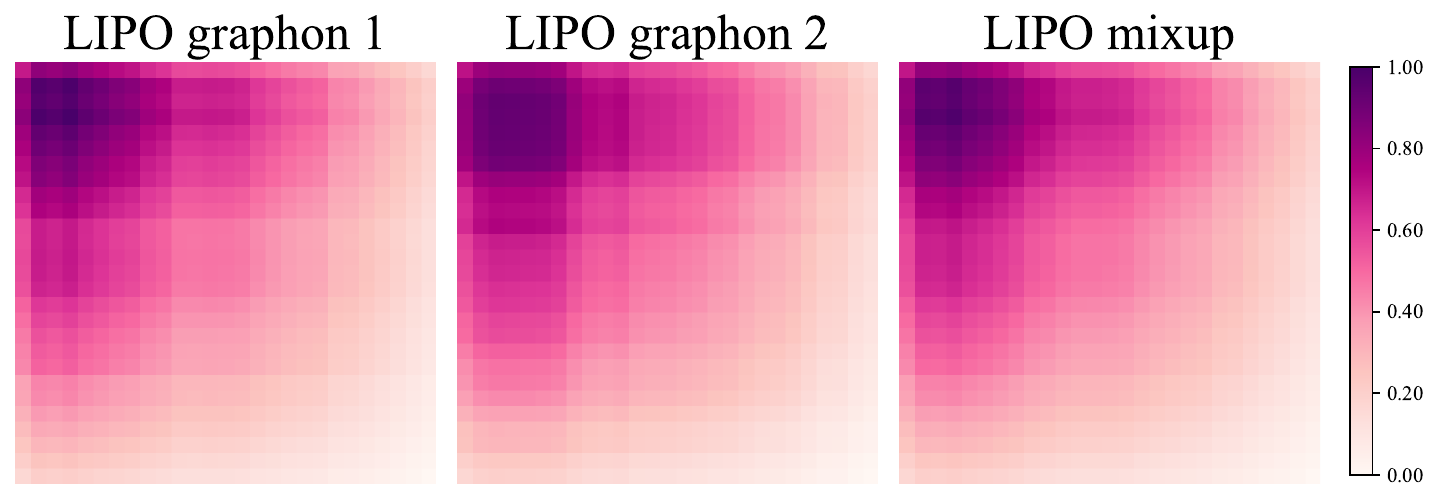}
\caption{Estimated graphons and their mixup results on the ClinTox/LIPO.}
  \end{minipage}
  \hfill
  \begin{minipage}[b]{0.48\textwidth}
    \centering
\includegraphics[width=\linewidth]{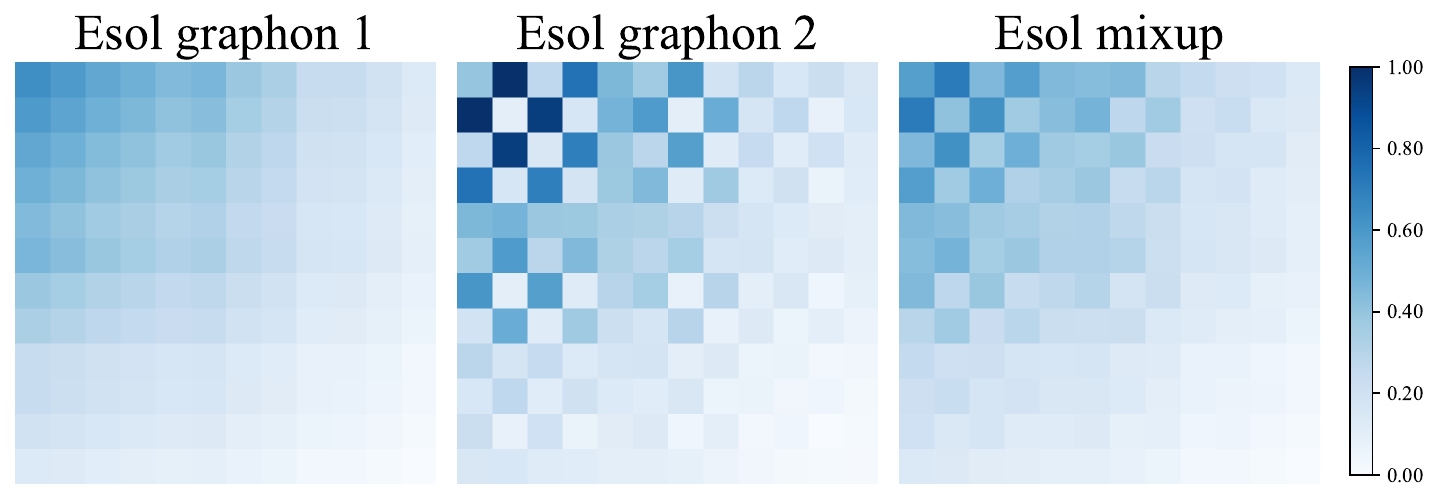} \\
\includegraphics[width=\linewidth]{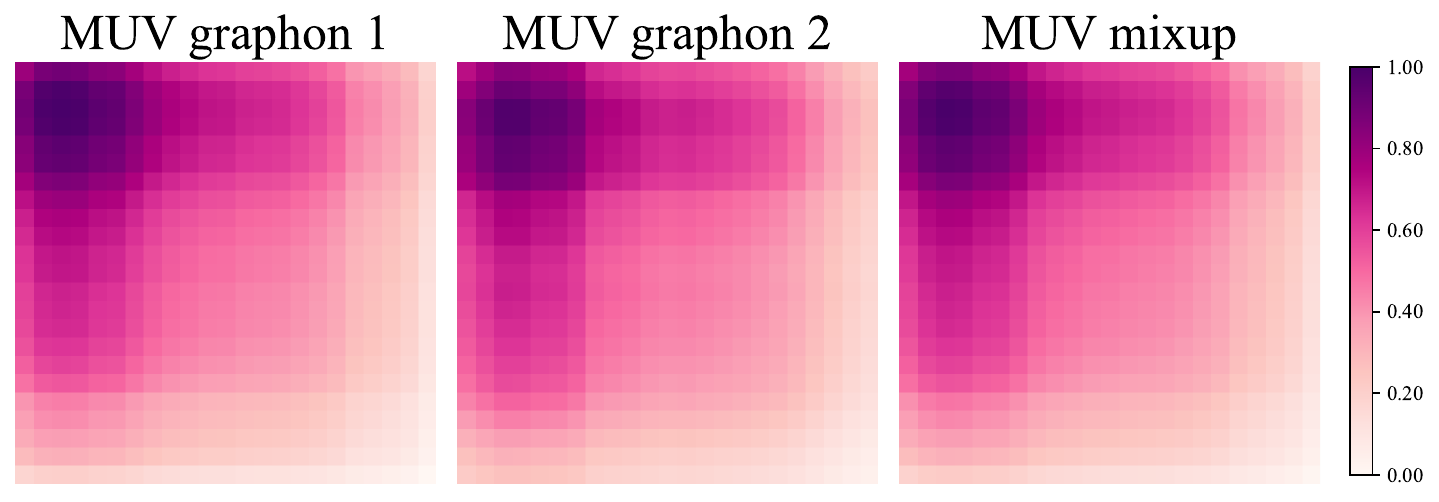}
\caption{Estimated graphons and their mixup results on the Esol/MUV.}
\label{fig:on-3}
  \end{minipage}
\end{figure*}

\end{document}